\documentclass[preprint,11pt]{elsarticle}

\usepackage{amssymb}

\usepackage{times}
\usepackage{soul}
\usepackage{url}
\usepackage[hidelinks]{hyperref}
\usepackage[utf8]{inputenc}
\usepackage[small]{caption}
\usepackage{graphicx}
\usepackage{amsmath}
\usepackage{amsthm}
\usepackage{booktabs}
\usepackage{algorithm}
\usepackage{algorithmic}
\usepackage{bm}
\usepackage{stmaryrd}
\usepackage{amssymb}
\usepackage{mathtools}
\DeclareMathOperator*{\argmax}{arg\,max}
\usepackage[whole]{bxcjkjatype}
\usepackage{xcolor}
\usepackage[whole]{bxcjkjatype}
\usepackage{bussproofs}

\newtheorem{example}{Example}
\newtheorem{definition}{Definition}
\newtheorem{theorem}{Theorem}

\newtheorem{proposition}{Proposition}
\newtheorem{corollary}{Corollary}

\newcommand{\m}[1]{[\![#1]\!]}

\newcommand{\ms}[1]{[\![#1]\!]}
\newcommand{\pms}[1]{[\![\![#1]\!]\!]}
\newcommand{\ams}[1]{(\!(#1)\!)}

\newcommand{\pams}[1]{(\!(\!(#1)\!)\!)}
\newcommand{\cent}{\mathrel{\scalebox{1}[1.5]{$\shortmid$}\mkern-3.1mu\raisebox{0.1ex}{$=$}}}
\newcommand{\ent}{\mathrel{\scalebox{1}[1.5]{$\shortmid$}\mkern-3.1mu\raisebox{0.1ex}{$\equiv$}}}

\newcommand{\vsim}{\mathrel{\scalebox{1}[1.5]{$\shortmid$}\mkern-3.1mu\raisebox{0.1ex}{$\sim$}}}

\journal{Artificial Intelligence}

\begin{document}

\begin{frontmatter}
            

\title{A Simple Generative Model of Logical Reasoning and Statistical Learning}

\author{Hiroyuki Kido}
\address{School of Computer Science and Informatics, Cardiff University,\\
Park Place, Cardiff, CF10 3AT, Wales, UK}

\begin{abstract}
Statistical learning and logical reasoning are two major fields of AI expected to be unified for human-like machine intelligence. Most existing work considers how to combine existing logical and statistical systems. However, there is no theory of inference so far explaining how basic approaches to statistical learning and logical reasoning stem from a common principle. Inspired by the fact that much empirical work in neuroscience suggests Bayesian (or probabilistic generative) approaches to brain function including learning and reasoning, we here propose a simple Bayesian model of logical reasoning and statistical learning. The theory is statistically correct as it satisfies Kolmogorov's axioms, is consistent with both Fenstad's representation theorem and maximum likelihood estimation and performs exact Bayesian inference with a linear-time complexity. The theory is logically correct as it is a data-driven generalisation of uncertain reasoning from consistency, possibility, inconsistency and impossibility. The theory is correct in terms of machine learning as its solution to generation and prediction tasks on the MNIST dataset is not only empirically reasonable but also theoretically correct against the K nearest neighbour method. We simply model how data causes symbolic knowledge in terms of its satisfiability in formal logic. Symbolic reasoning emerges as a result of the process of going the causality forwards and backwards. The forward and backward processes correspond to an interpretation and inverse interpretation in formal logic, respectively. The inverse interpretation differentiates our work from the mainstream often referred to as inverse entailment, inverse deduction or inverse resolution. The perspective gives new insights into learning and reasoning towards human-like machine intelligence.
\end{abstract}



\begin{highlights}
\item We propose a simple theory of inference to explain how basic approaches to logical reasoning and statistical learning stem from a common principle\footnote{A graphical abstract is in \ref{grab}.}.
\item The theory models how data causes symbolic knowledge. Symbolic reasoning emerges as a result of going the causality forwards and backwards.
\item The theory gives the logic community a fully data-driven logical-reasoning method with a linear-time complexity and its generalisation for reasoning from inconsistency and impossibility. 
\item The theory gives the machine learning community a new fully non-parametric all-nearest-neighbour method and its refinement for overfitting mitigation.
\end{highlights}

\end{frontmatter}

\section{Introduction}\label{sec:introduction}
\subsection{Background}
There is growing evidence that the brain is a generative model of environments. The image shown in Figure \ref{fig:illusions} would make one perceive a white triangle on the three black circles and one white triangle. A well-accepted explanation of the illusion is that our brains are trained to unconsciously use past experience to see what is likely to happen. The image would come as just a surprise if the sensory information eventually suppresses the prediction. By contrast, many illusions including the ones in Figure \ref{fig:illusions} cause an unusual situation where the prediction keeps suppressing the sensory information. Illusions tell us the importance of prior expectations in human perception.
\par
Much empirical work suggests Bayesian (i.e., probabilistic generative) models as an appropriate computational approach to reconcile (top-down) prediction and (bottom-up) sensory information in perception. Knill \cite{Knill:96} says `perception as Bayesian inference', and Hohwy \cite{Hohwy:14} says `there is converging evidence that the brain is a Bayesian mechanism'. Free-energy principle \cite{friston:10} uses a variational Bayesian method to account not only for perception but for human action. According to Friston \cite{friston:10}, Bayesian brain hypothesis \cite{Knill:04} is `the idea that the brain uses internal probabilistic (generative) models to update posterior beliefs, using sensory information, in an (approximately) Bayes-optimal fashion', and predictive coding \cite{Rao:99,Caucheteux:23} is `a tool used in signal processing for representing a signal using a linear predictive (generative) model'. Bayes' theorem derived from probability theory tells how the belief from past experience ought to be updated in light of sensory inputs. The mutual information (or Kullback-Leibler (KL) divergence) between the prior and posterior distributions is known as the Bayesian surprise \cite{Itti:09}, which is a measure of how surprising the sensory inputs are. Computational psychiatry \cite{Adams:16,Pellicano:12} uses Bayesian models to explain several symptoms of mental disorders such as schizophrenia and autism.
\par
The success of Bayesian models of brain function makes us think that there is a Bayesian model of how people perform logical reasoning, in a broad sense, including not only deductive reasoning but ampliative reasoning. Such a view would allow us to see commonsense reasoning, for instance, as a reconciliation between top-down prediction and bottom-up sensory information, just as the illusions shown in Figure \ref{fig:illusions} can be seen as a commonsense perception. This view of linking logical reasoning with perception is consistent with Mountcastle's discovery \cite{Mountcastle:82} summarised by Hawkins \cite{Hawkins:21}. Hawkins writes that `every part of the neocortex works on the same principle and that all the things we think of as intelligence\textemdash from seeing, to touching, to language, to high-level thought\textemdash are fundamentally the same'. This is evidenced by the experiment result \cite{Melchner:00} that ferrets learn to see with their eyes rewired to the auditory cortex and to hear with their ears rewired to the visual cortex.
\par
\begin{figure}[t]
\begin{center}
 \includegraphics[scale=0.3]{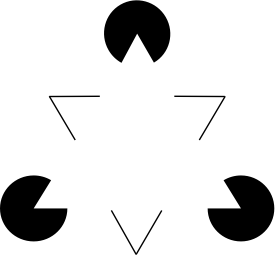}
   \caption{The Kanizsa triangle illusion. It looks as if a white triangle overlays the other objects. Adapted from Kanizsa \cite{kanizsa:55}.}
  \label{fig:illusions}
  \end{center}
\end{figure}
\subsection{Research Methods and Findings}
All the above discussions motivate us to ask how formal logic, as the laws of human thought, can be seen in terms of Bayesian models. To answer this question, this paper gives a simple Bayesian theory of formal inference for fully logical reasoning fully from data. Ordinary formal logics consider an interpretation on each model (denoted by $m$), which represents a state of the world\footnote{We use the term `model' to refer to a state of the world, e.g., each row of a truth table in propositional logic, and `interpretation' to refer to a truth-value assignment to a sentence.}. The interpretation defined on each model is a function that maps each formula (denoted by $\alpha$) to a truth value, which represents knowledge of the world. Given data (denoted by $d$), the most basic idea introduced in this paper is to consider the model and interpretation as likelihoods $p(m|d)$ and $p(\alpha|m)$, respectively. The model likelihood, i.e., $p(m|d)$, is defined to represent whether the data is an instance of the model, and the interpretation likelihood, i.e., $p(\alpha|m)$, is defined to represent whether the model satisfies the sentence (or equivalently, the sentence is true in the model). The two types of likelihoods will be used to formalise the following probabilistic process of how data causes symbolic knowledge via models in formal logic.
\begin{align}
p(\alpha)&=\sum_{n}p(\alpha|m_{n})p(m_{n})\nonumber\\
&=\sum_{n}p(\alpha|m_{n})\sum_{k}p(m_{n}|d_{k})p(d_{k})\label{sec1:eq1}
\end{align}
The first line shows how the probability distribution over the truth values of the sentence $\alpha$ is generated from the probability distribution over models in formal logic. The second line additionally shows how the probability distribution over the models is generated from the probability distribution over data. The hierarchy shown in Figure \ref{fig:hierarchy} illustrates the generative process of the calculation.
\begin{figure}[t]
\begin{center}
 \includegraphics[scale=0.4]{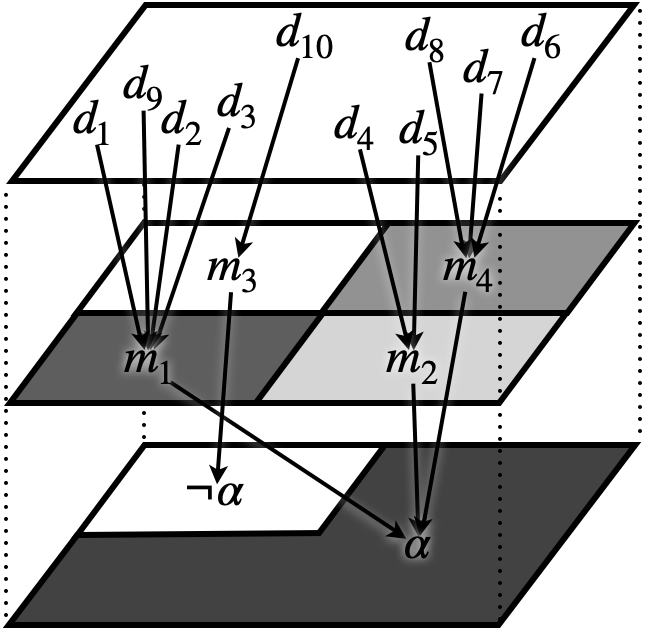}
  \caption{A schematic diagram of how data generates symbolic knowledge. The top layer is a probability distribution over data, the middle layer is a probability distribution over models in formal logic, and the bottom layer is a probability distribution over symbolic knowledge. A darker colour indicates a higher probability. A lower layer is an abstraction (i.e., selective ignorance) of the upper layer.}
  \label{fig:hierarchy}
  \end{center}
\end{figure}
\par
Let $\alpha$ be a sentence and $\Delta$ be a set of sentences. We will show symbolic reasoning from $\Delta$ to $\alpha$ is modelled as a conditional probability expanded as follows.
\begin{align}\label{sec1:eq2}
p(\alpha|\Delta)&=\sum_{n}p(\alpha|m_{n})p(m_{n}|\Delta)=\sum_{k}p(\alpha|d_{k})p(d_{k}|\Delta)
\end{align}
The second expression of Equation (\ref{sec1:eq2}) highlights the key characteristics of the idea introduced in this paper. It shows that symbolic reasoning is modelled as a process of going the interpretation of sentences forwards and backwards, without assuming dependency between the sentences. Indeed, we will show that $p(\alpha|m_{n})$ can represent the probability that $\alpha$ is true in $m_{n}$, i.e., an interpretation, and $p(m_{n}|\Delta)$ the probability that the model making all the sentences in $\Delta$ true is $m_{n}$, i.e., an inverse problem of the interpretation. We refer to the inverse problem as an \emph{inverse interpretation} to highlight the difference from the inverse entailment \cite{Muggleton:95}, inverse resolution \cite{Muggleton:88,Nienhuys:97} and inverse deduction \cite{Russell:20,Domingo:15}, which mainly concern modelling of dependency between sentences. The third expression of Equation (\ref{sec1:eq2}) shows that the model-driven reasoning can be fully data-driven. It shows that, in contrast to $p(\alpha)=\sum_{k}p(\alpha|d_{k})p(d_{k})$, the probability of $\alpha$ is obtained by referring to the posterior distribution over data updated by $\Delta$.
\par
We refer to the theory of inference as \emph{generative logic} (GL) as it models the process of how data generates symbolic knowledge via an interpretation in formal logic. In this paper, we look at statistical, logical and machine learning aspects of the GL to justify its correctness. From statistical point of view, we show
\begin{itemize}
\item it satisfies the Kolmogorov's axioms \cite{Kolmogorov:50}, which are basic rules any probability assignments ought to satisfy (see Proposition \ref{kolmogorov}),
\item it satisfies the Fenstad's theorem (Theorem 2.1 \cite{fenstad:67}), which relates the probability of a first-order sentence to the probability of the state of the world where the sentence is true (see Equation (\ref{eq:data-based})),
\item it is consistent with maximum likelihood estimation, which is the statistical method used most frequently to estimate a probability distribution only from data (see Equation (\ref{eq:data-driven})),
\item it replaces Bayes' theorem over symbolic knowledge with probabilistic reasoning from data (see Equations (\ref{eq:bayes})), and
\item it has a liner-time complexity with respect to the number of data (see Equations (\ref{eq:linear}) and (\ref{eq:linear-cp})).
\end{itemize}
\par
Once one has the view that the probability of any sentence is generated from data following Equation (\ref{sec1:eq1}), one can discuss possibility of sentences beyond their consistency. We say that a sentence(s) or model(s) is possible if its probability is non-zero and impossible otherwise. This leads to an alternative consequence relation, denoted by $\ent$, that ignores all impossible models in contrast to the classical consequence relation, denoted by $\cent$. From logical point of view, we show
\begin{itemize}
\item $p(\alpha|\Delta)=1$ if and only if $\Delta\cent\alpha$ when all models are possible and $\Delta$ is consistent (see Corollary \ref{cor:consistent_reasoning}), and
\item $p(\alpha|\Delta)=1$ if and only if $\Delta\ent\alpha$ when $\Delta$ is possible (see Corollary \ref{cor:possible_reasoning}).
\end{itemize}
These facts show that GL is a generalisation of $\cent$ and $\ent$ for reasoning with uncertainty (i.e, probability less than one) from consistency and possibility, respectively. From logical point of view, we further show 
\begin{itemize}
\item $p(\alpha|\Delta)=1$ if and only if $S\cent\alpha$, for all cardinality-maximal consistent subsets $S$ of $\Delta$, when all models are possible (see Corollary \ref{cor:paraconsistent_reasoning}), and
\item $p(\alpha|\Delta)=1$ if and only if $S\ent\alpha$, for all cardinality-maximal possible subsets $S$ of $\Delta$ (see Corollary \ref{cor:parapossible_reasoning}).
\end{itemize}
From machine-learning point of view, we finally show
\begin{itemize}
\item it solves a generation task on the MNIST dataset \cite{Deng:12} where a standard image of each digit is inferred from the digit. The result is not only intuitively reasonable but also statistically justified (see Figure \ref{fig:generation} and Equation (\ref{eq:generation})), and
\item it solves a prediction task on the MNIST dataset where the digit is inferred from a digit image. The result is not only empirically reasonable but also theoretically justified in terms of the K nearest neighbour method \cite{fix:51} (see Figures \ref{fig:LCs} and \ref{fig:ROCs}, and Equation (\ref{eq:prediction})). 
\end{itemize}
\subsection{Contributions and Related Work}\label{sec:discussion}
Statistical learning and logical reasoning are two major fields of AI expected to be unified for human-like machine intelligence. This paper contributes to statistical relational learning (SLR), which combines first-order logic and probability theory to solve problems across these fields. Acquiring knowledge from data and reasoning with the obtained knowledge are both essential processes of successful SLR systems. However, SLR approaches generally assume different algorithms for the two processes. Bayesian networks \cite{pearl:88} including naive Bayes, probabilistic logic programming (PLP) \cite{sato:95}, Markov logic networks (MLN) \cite{richardson:06}, probabilistic logic \cite{nilsson:86}, probabilistic relational models (PRM) \cite{friedman:99} and conditional probabilistic logic \cite{Rodder:00} assume both statistical and logical machineries. The statistical machinery assigns each logical sentence a probability value or weight so that it reflects aspects of the world, whereas the logical machinery performs logical reasoning on the probabilistic knowledge so that conclusions preserve the uncertainty of premises. For example, Bayesian networks, naive Bayes and PLP often assume maximum likelihood estimation or maximum a posteriori estimation for the statistical machinery. The probabilistic logic, conditional probabilistic logic and MLN often assume a human expert to play that role. However, some serious AI problems such as knowledge acquisition bottleneck, grounding, frame problems and commonsense reasoning \cite{Russell:20,Davis:15,Brewka:91} remain open without unifying the two machineries. In addition, Bayesian networks, PLP and MLN assume independence and/or conditional independence of knowledge or facts, which imposes each logical sentence to be fully independent or dependent only on a small number of other sentences for computational tractability. However, the independence and conditional independence rarely hold in real data. To the best of the author's knowledge, this paper for the first time shows that the statistical and logical machineries can be unified simply as probabilistic reasoning and that no independence or conditional independence is required for computational tractability. These facts are due to the novel idea that we only model the dependency between models and logical sentences, rather than the traditional idea that they model the dependency between logical sentences. This paper guarantees that reasoning becomes logically (and statistically as well) correct without the traditional idea.
\par
This paper also contributes to the field of formal logic by providing a simple theory of inference to reason fully logically fully from data with a linear-time complexity. In formal logic, the two major approaches to logical consequence relations are model checking and theorem proving \cite{Russell:20}. Reasoning with our theory falls into another category we obtain by generalising the model checking for a data-driven perspective. The time complexity of the data-driven reasoning is linear with respect to the number of data. This is in contrast to the model-based approach whose time complexity is theoretically exponential with respect to the number of symbols in propositional logic and is unbounded in predicate logic. Our theory deals with the \emph{inverse interpretation}, the inverse problem of the interpretation in formal logic, as opposed to the inverse entailment \cite{Muggleton:95}, inverse resolution \cite{Muggleton:88,Nienhuys:97} and inverse deduction \cite{Russell:20,Domingo:15}. The inverse interpretation provides a perspective to see reasoning as a reconciliation between top-down and bottom-up processing where stored data (or memory) restricts models to possible ones (see the top layer of Figure \ref{fig:hierarchy}) and observed knowledge further restricts the possible models (see the bottom layer).
\par
In addition, this paper contributes to the field of philosophical logic. Our theory unifies reasoning from consistency, possibility, inconsistency and impossibility. The reasoning from impossibility beyond possibility may be coined as \emph{parapossible reasoning} like reasoning from inconsistency beyond consistency is often referred to as paraconsistent reasoning \cite{Priest:02,Carnielli:07}. Parapossible reasoning systems may be more suitable than paraconsistent reasoning systems for a better interpretation of human reasoning. For example, inconsistency would not be an appropriate model of a conflict in human debate because almost all human reasoning is an enthymeme and thus no contradiction is usually made explicit in the debate. One would only say `$\alpha$: Tweety is a penguin' to attack the argument `$\beta$: Tweety flies because it is a bird.' However, $\{\alpha,\beta\}$ would be still consistent if $\alpha$ and $\beta$ are modelled with a usual formal language in a straightforward way. By contrast, $\{\alpha,\beta\}$ is probably impossible because data for both $\alpha$ and $\beta$ are unlikely to be observed simultaneously, i.e., $p(\{\alpha,\beta\})=0$. The long-standing conflict on the wave-particle duality of light can be better seen as impossibility rather than inconsistency.
\par
Finally, this paper contributes to the field of machine learning by providing new classification methods, which can theoretically be seen as a fully non-parametric all-nearest neighbour method and its reasonable refinement for overfitting mitigation. The latter method outperforms the K nearest neighbour method according to our experiment with the MNIST dataset. These promising results encourage the community to conduct further investigation against other methods with different datasets and to address machine-learning problems across data and symbolic knowledge.
\par
This paper is organised as follows. Section 2 defines a probabilistic model of how data causes symbolic knowledge via models in formal logic. In Sections 3 and 4, we theoretically analyse the statistical and logical properties of the probabilistic model, respectively. In Section 5, we generalise the discussion in Section 4 for a full  logical characterisation of the probabilistic model. In Section 6, the probabilistic model is applied to generation and prediction tasks on an image dataset to show its machine-learning correctness. Section 7 summarises the results of the paper.
\section{Generative Logic}
This section aims to give a mathematical representation of a probabilistic process of how data cause models in formal logic and the models cause symbolic knowledge. We begin by modelling how data cause models in formal logic. Let ${\cal D}=\{d_{1},d_{2},...,d_{K}\}$ be a multiset of data about states of the world. $D$ is a random variable whose realisations are data in ${\cal D}$. For all data $d_{k}\in{\cal D}$, we define the probability of $d_{k}$, as follows.
\begin{eqnarray*}
p(D=d_{k})=\frac{1}{K}
\end{eqnarray*}
$L$ represents a propositional or first-order language. For the sake of simplicity, we assume no function symbol or open formula in $L$. ${\cal M}=\{m_{1},m_{2},...,m_{N}\}$ is a set of models of $L$. ${\cal D}$ is assumed to be complete with respect to ${\cal M}$, and thus each data in ${\cal D}$ belongs to a single model in ${\cal M}$. $m$ is a function that maps each data to such a single model. $K_{n}$ denotes the number of data that belongs to $m_{n}$, i.e., $K_{n}=|\{d_{k}\in{\cal D}|m_{n}=m(d_{k})\}|$ where, for any set $X$, $|X|$ denotes the cardinality of $X$. $M$ is a random variable whose realisations are models in ${\cal M}$. For all models $m_{n}\in{\cal M}$ and data $d_{k}\in{\cal D}$, we define the conditional probability of $m_{n}$ given $d_{k}$, as follows.
\begin{eqnarray*}
&&p(M=m_{n}|D=d_{k})=
\begin{cases}
1 & \text{if } m_{n}=m(d_{k})\\
0 & \text{otherwise}
\end{cases}
\end{eqnarray*}
\par
The second task is to give a probabilistic representation of the process of how models cause the truth values of logical sentences. Ordinary formal logic considers an interpretation on each model.\footnote{In this paper, `model' means a model of a state of the world, whereas `interpretation' means an interpretation of a sentence.} The interpretation is a function that maps each formula to a truth value, which represents knowledge of the world. We here introduce parameter $\mu\in[0,1]$ to represent the extent to which each model is taken for granted in the interpretation. Concretely, $\mu$ denotes the probability that a formula is interpreted as being true (resp. false) in a model where it is true (resp. false). $1-\mu$ is therefore the probability that a formula is interpreted as being true (resp. false) in a model where it is false (resp. true). We will later see that interesting discussions emerge with $\mu$ approaching 1, i.e., $\mu\to 1$, rather than $\mu=1$. We assume that each formula is a random variable whose realisations are 0 and 1, denoting false and true, respectively. For all models $m_{n}\in{\cal M}$ and formulas $\alpha\in L$, we define the conditional probability of each truth value of $\alpha$ given $m_{n}$, as follows.
\begin{eqnarray*}
&&p(\alpha=1|M=m_{n})=
\begin{cases}
\mu & \text{if } m_{n}\in\llbracket\alpha=1\rrbracket\\
1-\mu & \text{otherwise }
\end{cases}
\\
&&p(\alpha=0|M=m_{n})=
\begin{cases}
\mu & \text{if } m_{n}\in\llbracket\alpha=0\rrbracket\\
1-\mu & \text{otherwise }
\end{cases}
\end{eqnarray*}
Here, $\llbracket\alpha=1\rrbracket$ denotes the set of all models in which $\alpha$ is true, and $\llbracket\alpha=0\rrbracket$ the set of all models in which $\alpha$ is false. The above expressions can be simply written as a Bernoulli distribution with parameter $\mu\in[0,1]$, i.e.,
\begin{eqnarray*}
p(\alpha|M=m_{n})=\mu^{\llbracket\alpha\rrbracket_{m_{n}}}(1-\mu)^{1-\llbracket\alpha\rrbracket_{m_{n}}}.
\end{eqnarray*}
Here, $\llbracket\alpha\rrbracket_{m_{n}}$ is a function such that $\llbracket\alpha\rrbracket_{m_{n}}=1$ if $m_{n}\in\llbracket\alpha\rrbracket$ and $\llbracket\alpha\rrbracket_{m_{n}}=0$ otherwise. Recall that $\alpha$ is a random variable, and thus $\llbracket\alpha\rrbracket_{m_{n}}$ is either $\llbracket\alpha=0\rrbracket_{m_{n}}$ or $\llbracket\alpha=1\rrbracket_{m_{n}}$.
\par
In classical logic, given a model, the truth value of each formula is independently determined. In probability theory, this means that the truth values of any two formulas $\alpha_{1}$ and $\alpha_{2}$ are conditionally independent given a model $m_{n}$, i.e., $p(\alpha_{1},\alpha_{2}|M=m_{n})=p(\alpha_{1}|M=m_{n})p(\alpha_{2}|M=m_{n})$. Note that the conditional independence holds not only for atomic formulas but for compound formulas as well. However, independence $p(\alpha_{1},\alpha_{2})= p(\alpha_{1})p(\alpha_{2})$ generally holds for neither atomic formulas nor compound formulas. Let $\Gamma=\{\alpha_{1},\alpha_{2},...,\alpha_{J}\}$ be a multiset of $J$ formulas. We thus have
\begin{eqnarray*}
p(\Gamma|M=m_{n})=\prod_{j=1}^{J}p(\alpha_{j}|M=m_{n}).
\end{eqnarray*}
\par
Thus far, we have defined $p(D)$ and $p(M|D)$ as categorical distributions and $p(\Gamma|M)$ as Bernoulli distributions with parameter $\mu$. Given a value of the parameter $\mu$, they provide the full joint distribution over all of the random variables, i.e., $p(\Gamma,M,D;\mu)$. We refer to this as a generative logic (GL). We will omit $\mu$ if it is clear from the context. In sum, the generative logic defines a data-driven interpretation by which the truth values of formulas are logically interpreted and probabilistically generated from models. The models are also probabilistically generated from data observed from the real world. The following proposition guarantees that GL assigns each formula a correct probability, regardless of the value of $\mu$.
\begin{proposition}\label{kolmogorov}
Let $p(\Gamma,M,D;\mu)$ be a generative logic and $\alpha,\beta\in\Gamma$. The generative logic satisfies the following Kolmogorov's axioms.
\begin{enumerate}
\item $0\leq p(\alpha=i)$ holds, for all $i\in\{0,1\}$.
\item $\sum_{i\in\{0,1\}}p(\alpha=i)=1$ holds.
\item $p(\alpha\lor\beta=i)=p(\alpha=i)+p(\beta=i)-p(\alpha\land\beta=i)$ holds, for all $i\in\{0,1\}$.
\end{enumerate}
\end{proposition}
\begin{proof}
See \ref{proof}.
\end{proof}
\begin{proposition}\label{negation}
Let $\alpha\in L$. $p(\alpha=0)=p(\neg\alpha=1)$ holds.
\end{proposition}
\begin{proof}
See \ref{proof}.
\end{proof}
In the following, we therefore replace $\alpha=0$ by $\lnot\alpha=1$ and then abbreviate $\lnot\alpha=1$ to $\lnot\alpha$. We also abbreviate $M=m_{n}$ to $m_{n}$ and $D=d_{k}$ to $d_{k}$. 
\begin{table}[t]
\begin{minipage}[c]{.45\hsize}
\centering
\setlength{\tabcolsep}{0.5mm} 
\caption{Models and data.}
\label{ex:data}
\begin{tabular}{c|cc|c}
& $rain$ & $wet$ & data ${\cal D}$\\\hline
$m_{1}$ & $0$ & $0$ & $d_{1},d_{2},d_{3},d_{4}$\\
$m_{2}$ & $0$ & $1$ & $d_{5},d_{6}$\\
$m_{3}$ & $1$ & $0$ & $d_{7}$\\
$m_{4}$ & $1$ & $1$ & $d_{8},d_{9},d_{10}$
\end{tabular}
\end{minipage}
\begin{minipage}[c]{.45\hsize}
\centering
\setlength{\tabcolsep}{0.5mm} 
\caption{Likelihoods.}
\label{ex:likelihoods}
\begin{tabular}{c|cc}
& $p(rain|M)$ & $p(wet|M)$\\\hline
$m_{1}$ & $1-\mu$ & $1-\mu$\\
$m_{2}$ & $1-\mu$ & $\mu$\\
$m_{3}$ & $\mu$ & $1-\mu$\\
$m_{4}$ & $\mu$ & $\mu$
\end{tabular}
\end{minipage}
\end{table}
\begin{example}\label{ex:BE}
Let $rain$ and $wet$ be two propositional symbols meaning `it is raining' and `the grass is wet,' respectively. Each row of Table \ref{ex:data} shows a different model, i.e., valuation. The last column shows how many data belongs to each model. Table \ref{ex:likelihoods} shows the likelihoods of the atomic propositions being true given a model. Given the GL $p(\Gamma,M,D;\mu=1)$, we have
\begin{align*}
p(rain|wet)&=\frac{\sum_{n=1}^{N}p(rain|m_{n})p(wet|m_{n})\sum_{k=1}^{K}p(m_{n}|d_{k})p(d_{k})}{\sum_{n=1}^{N}p(wet|m_{n})\sum_{k=1}^{K}p(m_{n}|d_{k})p(d_{k})}\\
&=\frac{\sum_{n=1}^{N}p(rain|m_{n})p(wet|m_{n})\frac{K_{n}}{K}}{\sum_{n=1}^{N}p(wet|m_{n})\frac{K_{n}}{K}}\\
&=\frac{(1-\mu)^{2}\frac{4}{10}+(1-\mu)\mu\frac{2}{10}+\mu(1-\mu)\frac{1}{10}+\mu^{2}\frac{3}{10}}{(1-\mu)\frac{4}{10}+\mu\frac{2}{10}+(1-\mu)\frac{1}{10}+\mu\frac{3}{10}}=\frac{\frac{3}{10}}{\frac{2}{10}+\frac{3}{10}}=\frac{3}{5}.
\end{align*}
\end{example}
\begin{example}
Suppose that $L$ has only one 2-ary predicate symbol `$blames$' and that the Herbrand universe for $L$ has only two constants $\{a,b\}$. There are four ground atoms, $\{blames(a,a), blames(a,b)$, $blames(b,a)$, $blames(b,b)\}$, which result in $2^{4}=16$ possible models. Each row of Table \ref{tab:FOL} shows a different model and the last column shows the number of data that belongs to the model. Models without data are abbreviated from the table. Given the GL $p(\Gamma,M,D;\mu=1)$, we have
\begin{align*}
&p(\forall x~ blames(x,a)|\exists x~blames(x,a))=\frac{p(\forall x~ blames(x,a),\exists x~blames(x,a))}{p(\exists x~blames(x,a))}\\
&=\frac{\sum_{n=1}^{16}p(\forall x~ blames(x,a)|m_{n})p(\exists x~blames(x,a)|m_{n})\frac{K_{n}}{K}}{\sum_{n=1}^{16}p(\exists x~blames(x,a)|m_{n})\frac{K_{n}}{K}}\\
&=\frac{\mu(1-\mu)\frac{2}{10}+\mu^{2}\frac{3}{10}+(1-\mu)^{2}\frac{4}{10}}{\mu\frac{2}{10}+\mu\frac{3}{10}+(1-\mu)\frac{4}{10}}=\frac{\frac{3}{10}}{\frac{2}{10}+\frac{3}{10}}=\frac{3}{5}.
\end{align*}
The same result is derived using the GL with $\mu\to 1$, i.e., $p(\Gamma,M,D;\mu\to1)$. However, the following result can only be derived using GL with $\mu\to 1$. Indeed, GL with $\mu=1$ causes a probability undefined due to division by zero.
\begin{align*}
&p(blames(a,b),blames(b,a)|\lnot blames(a,a),\lnot blames(b,b))\\
&=\frac{\sum_{n=1}^{16}\prod_{\alpha\in\{blames(a,b),blames(b,a),\lnot blames(a,a),\lnot blames(b,b)\}}\lim_{\mu\to 1}p(\alpha|m_{n})\frac{K_{n}}{K}}{\sum_{n=1}^{16}\prod_{\alpha\in\{\lnot blames(a,a),\lnot blames(b,b)\}}\lim_{\mu\to 1}p(\alpha|m_{n})\frac{K_{n}}{K}}\\
&=\lim_{\mu\to 1}\frac{(1-\mu)^{4}\frac{2}{10}+(1-\mu)\mu^{3}\frac{3}{10}+\mu^{2}(1-\mu)^{2}\frac{4}{10}}{(1-\mu)^{2}\frac{2}{10}+(1-\mu)\mu\frac{3}{10}+\mu(1-\mu)\frac{4}{10}}\\
&=\lim_{\mu\to 1}\frac{(1-\mu)^{3}\frac{2}{10}+\mu^{3}\frac{3}{10}+\mu^{2}(1-\mu)\frac{4}{10}}{(1-\mu)\frac{2}{10}+\mu\frac{3}{10}+\mu\frac{4}{10}}=\frac{\frac{3}{10}}{\frac{3}{10}+\frac{4}{10}}=\frac{3}{7}
\end{align*}
\begin{table}[t]
\caption{Three predicate models and ten associated data.}
\centering
\begin{tabular}{c|cccc|c}
 & \multicolumn{4}{c|}{$blames$} &\\
 & $(a,a)$ & $(a,b)$ & $(b,a)$ & $(b,b)$ & data ${\cal D}$\\\hline
$m_{1}$ & 1 & 0 & 0 & 1 & $d_{1},d_{2}$\\
$m_{2}$ & 1 & 1 & 1 & 0 & $d_{3},d_{4},d_{5}$\\
$m_{3}$ & 0 & 1 & 0 & 1 & $d_{6},d_{7},d_{8},d_{9},d_{10}$\\
other & \multicolumn{4}{c|}{other} & no data
\end{tabular}
\label{tab:FOL}
\end{table}
\end{example}
\section{Statistical Correctness}
\subsection{Statistical Estimation}\label{sec:statistics}
Fenstad gives the representation theorem (Theorem 2 \cite{fenstad:67}) to discuss how one ought to correctly assign a probability to any first-order formula. Let $\alpha\in L$ and $m_{n}\in{\cal M}$. Given no function symbol or open formula, the theorem can have the following simpler form, where $m_{n}\models\alpha$ represents that $m_{n}$ satisfies $\alpha$.
\begin{eqnarray}\label{eq:fenstad}
p(\alpha)=\sum_{n=1: m_{n}\models\alpha}^{N}p(m_{n})
\end{eqnarray}
The above theorem states that the probability of a formula is the sum of the probabilities of the models where the formula is true. When one has no prior knowledge about the probability of models, the most frequently used statistical method to estimate the probability only from data is maximum likelihood estimation, which is given as follows.
\begin{eqnarray*}
p(M)=\argmax_{\Phi}p({\cal D}|\Phi)
\end{eqnarray*}
Here, $\Phi=(\phi_{1},\phi_{2},...,\phi_{N})$ is the parameter of the categorical distribution $p(M)$ where $\phi_{N}=1-\phi_{1}-\phi_{2}-\cdots-\phi_{N-1}$ and $N$ is the number of realisations of $M$. $p(M)$ is thus defined as the parameter $\Phi$ maximising the likelihood $p({\cal D}|\Phi)$ of the data ${\cal D}$. As usual, assuming that each data is independent given $\Phi$, we have
\begin{eqnarray*}
p({\cal D}|\Phi)=\prod_{k=1}^{K}p(d_{k}|\Phi)=\phi_{1}^{K_{1}}\phi_{2}^{K_{2}}\cdots\phi_{N-1}^{K_{N-1}}(1-\phi_{1}-\phi_{2}-\cdots-\phi_{N-1})^{K_{N}}.
\end{eqnarray*}
$\Phi$ maximises the likelihood if and only if it maximises the log likelihood, which is given as follows.
\begin{eqnarray*}
L(\Phi)&=&K_{1}\log\phi_{1}+K_{2}\log\phi_{2}+\cdots+K_{N-1}\log\phi_{N-1}\\
&&+K_{N}\log(1-\phi_{1}-\phi_{2}-\cdots-\phi_{N-1})
\end{eqnarray*}
The maximum likelihood estimate is obtained by solving the following simultaneous equations, which are obtained by differentiating the log likelihood with respect to each $\phi_{n}(1\leq n\leq N-1)$.
\begin{eqnarray*}
\frac{\partial L(\Phi)}{\partial \phi_{n}}=\frac{K_{n}}{\phi_{n}}-\frac{K_{N}}{1-\phi_{1}-\phi_{2}-\cdots-\phi_{N-1}}=0
\end{eqnarray*}
The following is the solution to the simultaneous equations.
\begin{eqnarray*}
\Phi=\left(\frac{K_{1}}{K},\frac{K_{2}}{K},...,\frac{K_{N}}{K}\right)
\end{eqnarray*}
Therefore, the maximum likelihood estimate for the $n$-th model is just the ratio of the number of data in the model to the total number of data. Combining Equation (\ref{eq:fenstad}) and the maximum likelihood estimate, we have
\begin{eqnarray}\label{eq:fenstad+ML}
p(\alpha)=\sum_{n=1: m_{n}\models\alpha}^{N}\frac{K_{n}}{K}.
\end{eqnarray}
\par
Now, let $p(\Gamma,M,D;\mu)$ be a GL such that $\mu=1$ or $\mu\to 1$. $\mu\to 1$ means that $\mu$ approaches $1$. We show that both Fenstad's representation theorem and maximum likelihood estimation justify the GL. The representation theorem justifies the GL because probabilistic inference with the GL satisfies Equation (\ref{eq:fenstad}). For example, given $\mu\to 1$, the equation can be derived as follows.
\begin{align}\label{eq:data-based}
p(\alpha)&=\sum_{n=1}^{N}p(\alpha,m_{n})=\sum_{n=1}^{N}p(\alpha|m_{n})p(m_{n})=\sum_{n=1}^{N}\lim_{\mu\to 1}\mu^{\m{\alpha}_{m_{n}}}(1-\mu)^{1-\m{\alpha}_{m_{n}}}p(m_{n})\nonumber\\
&=\sum_{n=1}^{N}1^{\m{\alpha}_{m_{n}}}0^{1-\m{\alpha}_{m_{n}}}p(m_{n})=\sum_{n=1}^{N}\llbracket\alpha\rrbracket_{m_{n}}p(m_{n})=\sum_{n=1:m_{n}\in\llbracket\alpha\rrbracket}^{N}p(m_{n}).
\end{align}
Obviously, $\mu=1$ gives the same result. Maximum likelihood estimation also justifies the GL because probabilistic inference with the GL satisfies Equation (\ref{eq:fenstad+ML}).
\begin{align}
p(\alpha)&=\sum_{n=1}^{N}\sum_{k=1}^{K}p(\alpha,m_{n},d_{k})=\sum_{n=1}^{N}p(\alpha|m_{n})\sum_{k=1}^{K}p(m_{n}|d_{k})p(d_{k})\label{eq:data-driven0}\\
&=\sum_{n=1}^{N}\llbracket\alpha\rrbracket_{m_{n}}\frac{K_{n}}{K}=\sum_{n=1: m_{n}\in\llbracket\alpha\rrbracket}^{N}\frac{K_{n}}{K}\label{eq:data-driven}
\end{align}
We have shown that the GL with $\mu=1$ or $\mu\to 1$ not only follows Fenstad's representation theorem and maximum likelihood estimation but also derives them as probabilistic reasoning in a unified way. When $p(\Delta)\neq 0$, the conditional probability of $\alpha\in L$ given $\Delta\subseteq L$ is thus derived from Equation (\ref{eq:data-driven}) as follows.
\begin{align}\label{eq:data-driven-cp}
p(\alpha|\Delta)=\frac{p(\alpha,\Delta)}{p(\Delta)}=\frac{\sum_{n=1: m_{n}\in\llbracket\alpha,\Delta\rrbracket}^{N}p(m_{n})}{\sum_{n=1: m_{n}\in\llbracket\Delta\rrbracket}^{N}p(m_{n})}=\frac{\sum_{n=1: m_{n}\in\llbracket\alpha,\Delta\rrbracket}^{N}K_{n}}{\sum_{n=1: m_{n}\in\llbracket\Delta\rrbracket}^{N}K_{n}}
\end{align}
Namely, $p(\alpha|\Delta)$ is the sum of the maximum likelihood estimates of the models where $\alpha$ and $\Delta$ are true, divided by the sum of the maximum likelihood estimates of the models where $\Delta$ is true. It turns out to be the number of data in the models where $\alpha$ and $\Delta$ are true, divided by the number of data in the models where $\Delta$ is true.
\par
Bayesian inference over symbols, i.e., $p(\alpha|\Delta)=p(\Delta|\alpha)p(\alpha)/p(\Delta)$, can be replaced by reasoning between models and symbols. When $p(\Delta)\neq 0$, each term of the Bayes' theorem can be expanded with a GL as follows.
\begin{eqnarray*}
p(\alpha|\Delta)=\frac{\sum_{m}p(\alpha|m)p(\Delta|m)p(m)}{\sum_{m}p(\Delta|m)p(m)}&~~~~&p(\Delta|\alpha)=\frac{\sum_{m}p(\Delta|m)p(\alpha|m)p(m)}{\sum_{m}p(\alpha|m)p(m)}\\
p(\alpha)=\sum_{m}p(\alpha|m)p(m)&~~~~&p(\Delta)=\sum_{m}p(\Delta|m)p(m)
\end{eqnarray*}
The result below shows that the Bayes' theorem still holds after the replacement.
\begin{align}
\frac{p(\Delta|\alpha)p(\alpha)}{p(\Delta)}&=\frac{\frac{\sum_{m}p(\Delta|m)p(\alpha|m)p(m)}{\sum_{m}p(\alpha|m)p(m)}\sum_{m}p(\alpha|m)p(m)}{\sum_{m}p(\Delta|m)p(m)}\nonumber\\
&=\frac{\sum_{m}p(\alpha|m)p(\Delta|m)p(m)}{\sum_{m}p(\Delta|m)p(m)}=p(\alpha|\Delta)\label{eq:bayes}
\end{align}
All the results discussed in this section justify the correctness of the GL from a statistical point of view.
\subsection{Time Complexity}
In general, the time complexity of Equation (\ref{eq:data-driven0}) depends on $N$, the number of models, which is unbounded in predicate logic and exponentially increases in propositional logic with respect to the number of propositional symbols. However, the exponential complexity can be reduced as follows to a linear complexity with respect to the number of data.
\begin{align}\label{eq:linear}
p(\alpha)&=\sum_{n=1}^{N}\sum_{k=1}^{K}p(\alpha,m_{n},d_{k})=\sum_{k=1}^{K}p(d_{k})\sum_{n=1}^{N}p(\alpha|m_{n})p(m_{n}|d_{k})\nonumber\\
&=\sum_{k=1}^{K}p(d_{k})p(\alpha|m(d_{k}))=\frac{1}{K}\sum_{k=1}^{K}\mu^{\m{\alpha}_{m(d_{K+1})}}(1-\mu)^{1-\m{\alpha}_{m(d_{K+1})}}
\end{align}
Note that the above equation holds regardless of the value of $\mu$. $p(\alpha|m(d_{K+1}))=\llbracket\alpha\rrbracket_{m(d_{K+1})}$ holds if $\mu=1$ or $\mu\to 1$. It is obvious from Equation (\ref{eq:linear}) that a conditional probability over symbols has also a linear-time complexity.
\begin{align}\label{eq:linear-cp}
p(\alpha|\Delta)&=\frac{\sum_{k=1}^{K}p(\alpha|m(d_{k}))p(\Delta|m(d_{k}))}{\sum_{k=1}^{K}p(\alpha|m(d_{k}))}=\frac{\sum_{k=1}^{K}\prod_{x\in\Delta\cup\{\alpha\}}p(x|m(d_{k}))}{\sum_{k=1}^{K}p(\alpha|m(d_{k}))}
\end{align}
Figure \ref{fig:time_complexity} illustrates the difference between the naive approach with the summation over all models (shown on the left) and Equation (\ref{eq:linear-cp}) with the summation over all data. It is illustrated that the number of models is generally much larger than the number of data and that the presence of the three layers contributes to reduce the complexity of reasoning over symbols.
\begin{figure}[t]
\begin{center}
 \includegraphics[scale=0.4]{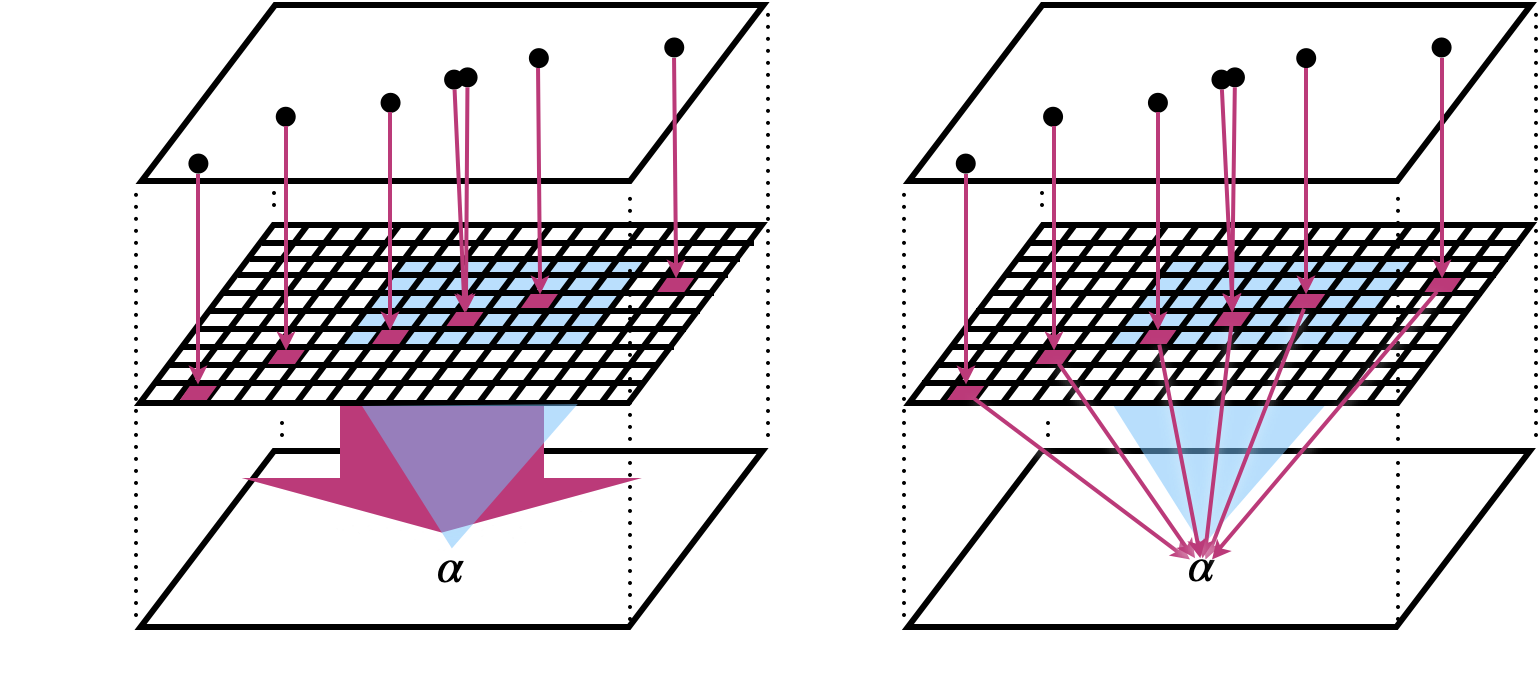}
  \caption{The hierarchy on the left illustrates the naive approach, e.g., Equation (\ref{eq:data-driven-cp}), where all the models are counted with the summation over models. The hierarchy on the right illustrates Equation (\ref{eq:linear-cp}) where all the models without data are ignored by the summation over data.}
  \label{fig:time_complexity}
  \end{center}
\end{figure}
\par
Equation (\ref{eq:data-driven}) has only a constant-time complexity for recalculation with new data. Let $p_{X}(\alpha)$ denote the probability of $\alpha$ calculated with $X$ data. $p_{K+1}(\alpha)$ can be calculated using $p_{K}(\alpha)$ as follows.
\begin{align}\label{eq:updating}
&p_{K+1}(\alpha)=\sum_{n=1}^{N}p(\alpha|m_{n})\sum_{k=1}^{K+1}p(m_{n}|d_{k})p(d_{k})\nonumber\\
&=\sum_{n=1}^{N}p(\alpha|m_{n})\sum_{k=1}^{K}p(m_{n}|d_{k})p(d_{k})+\sum_{n=1}^{N}p(\alpha|m_{n})p(m_{n}|d_{K+1})p(d_{K+1})\nonumber\\
&=\frac{K}{K+1}\sum_{n=1}^{N}p(\alpha|m_{n})\sum_{k=1}^{K}p(m_{n}|d_{k})\frac{1}{K}+\sum_{n=1}^{N}p(\alpha|m_{n})p(m_{n}|d_{K+1})\frac{1}{K+1}\nonumber\\
&=\frac{Kp_{K}(\alpha)+p(\alpha|m(d_{K+1}))}{K+1}
\end{align}
The equation below shows that $p(\alpha)$ converges with respect to the number of data.
\begin{align*}
\lim_{K\rightarrow\infty}p_{K+1}(\alpha)&=\lim_{K\rightarrow\infty}\frac{Kp_{K}(\alpha)+p(\alpha|m(d_{K+1}))}{K+1}\\
&=\lim_{K\rightarrow\infty}\frac{p_{K}(\alpha)+\frac{1}{K}p(\alpha|m(d_{K+1}))}{1+\frac{1}{K}}=p_{K}(\alpha)
\end{align*}
Equation (\ref{eq:updating}) implies that a conditional probability $p_{K+1}(\alpha|\beta)$ needs not only $p_{K}(\alpha|\beta)$ but $p_{K}(\beta)$ for recalculation with new data.
\begin{align*}
p_{K+1}(\alpha|\beta)&=\frac{p_{K+1}(\alpha,\beta)}{p_{K+1}(\beta)}=\frac{Kp_{K}(\alpha,\beta)+p(\alpha|m(d_{K+1}))p(\beta|m(d_{K+1}))}{Kp_{K}(\beta)+p(\beta|m(d_{K+1}))}\\
&=\frac{Kp_{K}(\alpha|\beta)p_{K}(\beta)+p(\alpha|m(d_{K+1}))p(\beta|m(d_{K+1}))}{Kp_{K}(\beta)+p(\beta|m(d_{K+1}))}
\end{align*}
\begin{example}
\begin{table}[t]
\centering
\setlength{\tabcolsep}{0.5mm} 
\caption{New data.}
\label{ex:update}
\begin{tabular}{c|cc|c|c}
& $bird$ & $fly$ & current data & new data\\\hline
$m_{1}$ & $0$ & $0$ & $d_{1},d_{2},d_{3},d_{4},d_{5}$ &\\
$m_{2}$ & $0$ & $1$ & $d_{6},d_{7}$ &\\
$m_{3}$ & $1$ & $0$ & & $d_{11}$\\
$m_{4}$ & $1$ & $1$ & $d_{8},d_{9},d_{10}$
\end{tabular}
\end{table}
Let $p(\Gamma,M,D;\mu=1)$ be a GL and $bird, fly\in\Gamma$ be propositional symbols meaning `it is a bird' and `it flies', respectively. Each row of Table \ref{ex:update} shows a different model. Given the ten data shown in the fourth column, the probability that $bird$ implies $fly$ is calculated using Equation (\ref{eq:linear}), as follows.
\begin{align*}
&p(bird\rightarrow fly)=\frac{1}{10}\sum_{k=1}^{10}\llbracket bird\rightarrow fly\rrbracket_{m(d_{k})}=1
\end{align*}
It is obvious from the GL that the counterintuitive knowledge that birds must fly comes from a lack of data. Indeed, taking into account the eleventh data shown in the last column, the probability is updated using Equation (\ref{eq:updating}), as follows.
\begin{eqnarray*}
p_{11}(\alpha)=\frac{10p_{10}(bird\rightarrow fly)+\llbracket bird\rightarrow fly\rrbracket_{m(d_{11})}}{11}=\frac{10}{11}
\end{eqnarray*}
\end{example}
\section{Logical Correctness}
\subsection{Consistent Reasoning}
We showed in the last section that, given a GL $p(\Gamma,M,D;\mu)$ such that $\mu=1$ or $\mu\to 1$, $p(M)$ is equivalent to the maximum likelihood estimate, i.e., for all $m_{n}\in{\cal M}$,
\begin{eqnarray*}
p(m_{n})=\sum_{k=1}^{K}p(m_{n}|d_{k})p(d_{k})=\frac{K_{n}}{K}.
\end{eqnarray*}
Therefore, given $\mu=1$ or $\mu\to 1$, $p(\Gamma,M,D;\mu)$ is equivalent to $p(\Gamma,M;\mu)$ if $p(M)$ is the maximum likelihood estimate. For the sake of simplicity, we also refer to the latter as a GL and use it without distinction.
\par
In this section, we look at the GL $p(\Gamma,M;\mu=1)$ on the assumption that every model is possible, i.e., $p(m)\neq 0$, for all models $m$, denoted by $0\notin p(M)$. Recall that a set $\Delta$ of formulas entails a formula $\alpha$ in classical logic, denoted by $\Delta\models\alpha$, iff $\alpha$ is true in every model in which $\Delta$ is true, i.e., $\llbracket\Delta\rrbracket\subseteq\llbracket\alpha\rrbracket$. The following theorem relates the probability of a formula to the probability of its models.
\begin{theorem}\label{thrm:consistent}
Let $p(\Gamma,M;\mu=1)$ be a generative logic such that $0\notin p(M)$, $\alpha\in\Gamma$ and $\Delta\subseteq\Gamma$.
\begin{eqnarray*}
p(\alpha|\Delta)=
\begin{cases}
\displaystyle{\frac{\sum_{m\in\ms{\Delta}\cap\ms{\alpha}}p(m)}{\sum_{m\in\ms{\Delta}}p(m)}}&\text{if }\ms{\Delta}\neq\emptyset\\
\text{undefined}&\text{otherwise}
\end{cases}
\end{eqnarray*}
\end{theorem}
\begin{proof}
Let $|\Delta|$ denote the cardinality of $\Delta$. Recall that, in formal logic, the fact that $\Delta$ has a model is equivalent to the fact that there is a model $m$ in which every formula in $\Delta$ is true in $m$. Dividing models into the models of $\Delta$ and the others, we have
\begin{eqnarray*}
p(\alpha|\Delta)&=&\frac{\sum_{m}p(\alpha|m)p(\Delta|m)p(m)}{\sum_{m}p(\Delta|m)p(m)}\\
&=&\frac{\displaystyle{\sum_{m\in\llbracket\Delta\rrbracket}p(m)p(\alpha|m)\mu^{|\Delta|}+\sum_{m\notin\llbracket\Delta\rrbracket}p(m)p(\alpha|m)p(\Delta|m)}}{\displaystyle{\sum_{m\in\llbracket\Delta\rrbracket}p(m)\mu^{|\Delta|}+\sum_{m\notin\llbracket\Delta\rrbracket}p(m)p(\Delta|m)}}.
\end{eqnarray*}
By definition, $p(\Delta|m)=\prod_{\beta\in\Delta}p(\beta|m)=\prod_{\beta\in\Delta}\mu^{\llbracket\beta\rrbracket_{m}}(1-\mu)^{1-{\llbracket\beta\rrbracket_{m}}}$. For all $m\notin\llbracket\Delta\rrbracket$, there is $\beta\in\Delta$ such that $\llbracket\beta\rrbracket_{m}=0$. Therefore, $p(\Delta|m)=0$ when $\mu=1$, for all $m\notin\llbracket\Delta\rrbracket$. We thus have
\begin{eqnarray*}
p(\alpha|\Delta)=\frac{\sum_{m\in\llbracket\Delta\rrbracket}p(m)p(\alpha|m)1^{|\Delta|}}{\sum_{m\in\llbracket\Delta\rrbracket}p(m)1^{|\Delta|}}=\frac{\sum_{m\in\llbracket\Delta\rrbracket}p(m)1^{\llbracket\alpha\rrbracket_{m}}0^{1-\llbracket\alpha\rrbracket_{m}}}{\sum_{m\in\llbracket\Delta\rrbracket}p(m)}.
\end{eqnarray*}
Since $1^{\llbracket\alpha\rrbracket_{m}}0^{1-\llbracket\alpha\rrbracket_{m}}=1^{1}0^{0}=1$ if $m\in\llbracket\alpha\rrbracket$ and $1^{\llbracket\alpha\rrbracket_{m}}0^{1-\llbracket\alpha\rrbracket_{m}}=1^{0}0^{1}=0$ if $m\notin\llbracket\alpha\rrbracket$, we have
\begin{align}\label{eq:cent_characterisation}
p(\alpha|\Delta)=\frac{\sum_{m\in\llbracket\Delta\rrbracket\cap\llbracket\alpha\rrbracket}p(m)}{\sum_{m\in\llbracket\Delta\rrbracket}p(m)}.
\end{align}
In addition, if $\ms{\Delta}=\emptyset$ then $p(\alpha|\Delta)$ is undefined due to division by zero.
\end{proof}
The following Corollary states that, if $\Delta$ is consistent, uncertain reasoning, i.e., reasoning with a probability less than one, over symbols with the GL is a generalisation of the classical consequence relation.
\begin{corollary}\label{cor:consistent_reasoning}
Let $p(\Gamma,M,D;\mu=1)$ be a generative logic such that $0\notin p(M)$, $\alpha\in\Gamma$ and $\Delta\subseteq\Gamma$ such that $\ms{\Delta}\neq\emptyset$. $p(\alpha|\Delta)=1$ iff (if and only if) $\Delta\cent\alpha$.
\end{corollary}
\begin{proof}
By assumption, we have $0\notin p(M)$. From Equation (\ref{eq:cent_characterisation}), $\frac{\sum_{m\in\llbracket\Delta\rrbracket\cap\llbracket\alpha\rrbracket}p(m)}{\sum_{m\in\llbracket\Delta\rrbracket}p(m)}=1$ iff $\llbracket\alpha\rrbracket\supseteq\llbracket\Delta\rrbracket$, i.e., $\Delta\models\alpha$.
\end{proof}
The following example shows that Theorem \ref{thrm:consistent} does not hold without the assumption $0\notin p(M)$. 
\begin{example}
Given $p(M)=(0.6,0,0.1,0.3)$ in Example \ref{ex:BE}, $p(rain|wet)=1$ but $\{wet\}\not\models rain$.
\end{example}
For any formula $\alpha$ and set $\Delta$ of formulas, $\llbracket\Delta\rrbracket=\emptyset$ implies $\Delta\models\alpha$ in classical logic as the classical entailment is defined as $\m{\Delta}\subseteq\m{\alpha}$. Thus, Theorem \ref{thrm:consistent} implies that if $p(\alpha|\Delta)=1$ then $\Delta\cent\alpha$, but not vice versa. In other words, certain reasoning, i.e., reasoning with a probability of one, with the GL is more cautious than the classical entailment.
\par
Figure \ref{fig:consistent_reasoning} illustrates consistent reasoning with the GL $p(\Gamma,M,D;\mu=1)$. It shows that models are constrained by formulas, but not restricted by data. It is indeed shown that every model has a probability of non-zero due to the assumption $0\notin p(M)$ regardless of the data presence.
\par
\begin{figure}[t]
\begin{center}
 \includegraphics[scale=0.4]{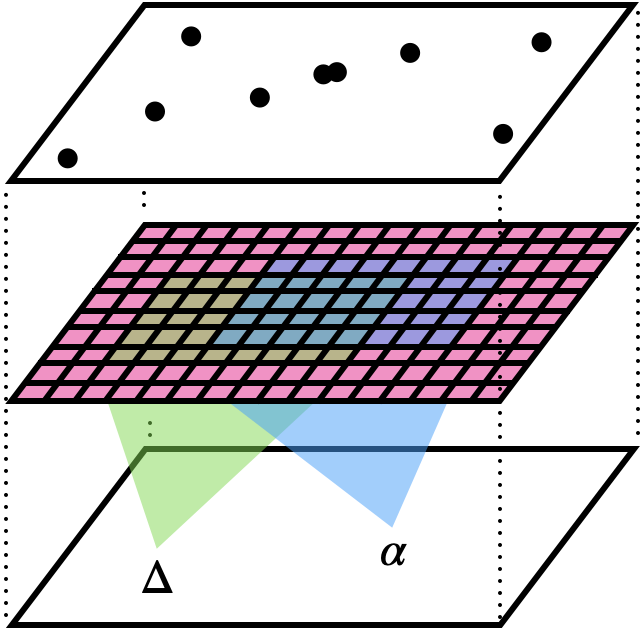}
  \caption{Consistent reasoning with the GL $p(\Gamma,M,D;\mu=1)$. A model (i.e., a cell in the second layer) is coloured in purple if its probability is non-zero. A model is coloured in green (resp. blue) if all the formulas in $\Delta\subseteq\Gamma$ are (resp. $\alpha\in\Gamma$ is) true in the model.}
  \label{fig:consistent_reasoning}
  \end{center}
\end{figure}
\subsection{Possible Reasoning}
In the previous section, we looked at the GL $p(\Gamma,M,D;\mu=1)$ and characterised its certain reasoning as reasoning from consistency with the classical entailment. The assumption we made for the characterisation is $0\notin p(M)$ meaning that no model has a probability of zero. In this section, we look at the same GL $p(\Gamma,M,D;\mu=1)$ without the assumption. This section thus aims to fully generalise the discussion of the previous section.
\par
We refer to a model, formula or set of formulas as being possible if its probability is non-zero, and impossible otherwise. For any $\Delta\subseteq L$, we use symbol $\pms{\Delta}$ to denote the set of all the possible models of $\Delta$, i.e., $\pms{\Delta}=\{m\in\m{\Delta}|p(m)\neq 0\}$. We assume $\pms{\Delta}_{m}=1$ if $m\in\pms{\Delta}$ and $\pms{\Delta}_{m}=0$ otherwise. Obviously, $\pms{\Delta}\subseteq\m{\Delta}$ holds, for all $\Delta\subseteq L$, and $\pms{\Delta}=\m{\Delta}$ holds if all models are possible. We define an alternative consequence relation based only on possible models.
\begin{definition}[Consequence]
Let $\Delta\subseteq L$ and $\alpha\in L$. $\alpha$ is a consequence of $\Delta$, denoted by $\Delta\ent\alpha$, if $\pms{\alpha}\supseteq\pms{\Delta}$.
\end{definition}
The alternative consequence relation $\ent$ needs to be distinguished from $\cent$, which is defined as $[\![\alpha]\!]\supseteq[\![\Delta]\!]$. We refer to $\cent$ as the classical consequence relation and $\ent$ as a consequence relation. The following proposition shows that the consequence relation is weaker than the classical consequence relation.
\begin{proposition}
Let $\Delta\subseteq L$ and $\alpha\in L$. If $\Delta\cent\alpha$ then $\Delta\ent\alpha$, but not vice versa.
\end{proposition}
\begin{proof}
($\Rightarrow$) Recall that $\Delta\cent\alpha$ iff $\ms{\Delta}\subseteq\ms{\alpha}$. For all $X\subseteq{\cal M}$, $\ms{\Delta}\setminus X\subseteq\ms{\alpha}\setminus X$ holds. ($\Leftarrow$) Suppose $\Delta$, $\alpha$ and $m$ such that $\ms{\Delta}=\ms{\alpha}\cup\{m\}$ and $p(m)=0$. We then have $\Delta\ent\alpha$, but $\Delta\not\cent\alpha$.
\end{proof}
For all models, the classical consequence relation requires all the models of the premises to be the models of the conclusion. In contrast, the consequence relation requires all the possible models of the premises to be the models of the conclusion. The following theorem relates the probability of a formula to the probability of its possible models.
\begin{theorem}\label{thrm:possible}
Let $p(\Gamma,M;\mu=1)$ be a generative logic, $\alpha\in\Gamma$ and $\Delta\subseteq\Gamma$.
\begin{eqnarray*}
p(\alpha|\Delta)=
\begin{cases}
\displaystyle{\frac{\sum_{m\in\pms{\Delta}\cap\pms{\alpha}}p(m)}{\sum_{m\in\pms{\Delta}}p(m)}}&\text{if }\pms{\Delta}\neq\emptyset\\
\text{undefined}&\text{otherwise}
\end{cases}
\end{eqnarray*}
\end{theorem}
\begin{proof}
Dividing models into the possible models $\pms{\Delta}$ and the others, we have
\begin{align*}
p(\alpha|\Delta)&=\frac{\sum_{m}p(\alpha|m)p(\Delta|m)p(m)}{\sum_{m}p(\Delta|m)p(m)}\\
&=\frac{\displaystyle{\sum_{\hat{m}\in\pms{\Delta}}p(\alpha|\hat{m})\mu^{|\Delta|}p(\hat{m})+\sum_{m\notin\pms{\Delta}}p(\alpha|m)p(\Delta|m)p(m)}}{\displaystyle{\sum_{\hat{m}\in\pms{\Delta}}\mu^{|\Delta|}p(\hat{m})+\sum_{m\notin\pms{\Delta}}p(\Delta|m)p(m)}}.
\end{align*}
$p(\Delta|m)=\prod_{\beta\in\Delta}p(\beta|m)=\prod_{\beta\in\Delta}\mu^{\m{\beta}_{m}}(1-\mu)^{1-{\m{\beta}_{m}}}$. Recall that $\pms{\Delta}=\{m\in\m{\Delta}|p(m)\neq0\}$. Thus, for all $m\notin\pms{\Delta}$, if $m\notin\m{\Delta}$ then there is $\beta\in\Delta$ such that $\m{\beta}_{m}=0$ and if $m\in\m{\Delta}$ then $p(m)=0$. Therefore, $p(\Delta|m)=0$ or $p(m)=0$ when $\mu=1$, for all $m\notin\pms{\Delta}$. We thus have
\begin{align*}
p(\alpha|\Delta)=\frac{\sum_{m\in\pms{\Delta}}p(\alpha|m)1^{|\Delta|}p(m)}{\sum_{m\in\pms{\Delta}}1^{|\Delta|}p(m)}=\frac{\sum_{m\in\pms{\Delta}}1^{\m{\alpha}_{m}}0^{1-\m{\alpha}_{m}}p(m)}{\sum_{m\in\pms{\Delta}}p(m)}.
\end{align*}
Since $1^{\m{\alpha}_{m}}0^{1-\m{\alpha}_{m}}=1^{1}0^{0}=1$ if $m\in\m{\alpha}$ and $1^{\m{\alpha}_{m}}0^{1-\m{\alpha}_{m}}=1^{0}0^{1}=0$ if $m\notin\m{\alpha}$, we have
\begin{align}\label{ijcai23:proof:2}
p(\alpha|\Delta)=\frac{\sum_{m\in\pms{\Delta}\cap\m{\alpha}}p(m)}{\sum_{m\in\pms{\Delta}}p(m)}=\frac{\sum_{m\in\pms{\Delta}\cap\pms{\alpha}}p(m)}{\sum_{m\in\pms{\Delta}}p(m)}.
\end{align}
In addition, if $\pms{\Delta}=\emptyset$ then $p(\alpha|\Delta)$ is undefined due to division by zero.
\end{proof}
The following Corollary states that, if $\Delta$ is possible, uncertain reasoning over symbols with the GL is a generalisation of the alternative consequence relation.
\begin{corollary}\label{cor:possible_reasoning}
Let $p(\Gamma,M;\mu=1)$ be a generative logic, $\alpha\in\Gamma$ and $\Delta\subseteq\Gamma$ such that $\pms{\Delta}\neq\emptyset$. $p(\alpha|\Delta)=1$ iff $\Delta\ent\alpha$.
\end{corollary}
\begin{proof}
Recall that $\Delta\ent\alpha$ iff $\pms{\alpha}\supseteq\pms{\Delta}$. Since Equation (\ref{ijcai23:proof:2}) and $p(m)\neq0$ hold, for all $m\in\pms{\Delta}$, $p(\alpha|\Delta)=1$ iff $\pms{\alpha}\supseteq\pms{\Delta}$.
\end{proof}
\par
Figure \ref{fig:possible_reasoning} illustrates possible reasoning with the GL $p(\Gamma,M,D;\mu=1)$. It shows that models are not only constrained by formulas but also restricted by data. In contrast to consistent reasoning with the GL, possible reasoning ignores all the models without data.
\par
\begin{figure}[t]
\begin{center}
 \includegraphics[scale=0.4]{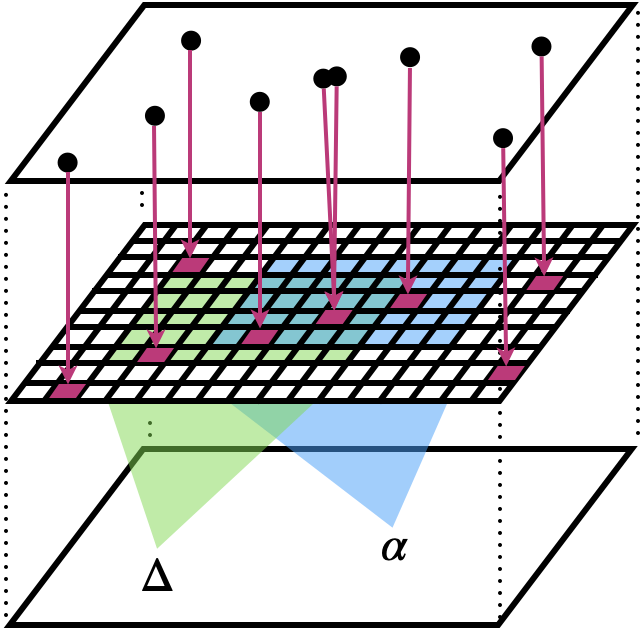}
  \caption{Possible reasoning with the GL $p(\Gamma,M,D;\mu=1)$.}
  \label{fig:possible_reasoning}
  \end{center}
\end{figure}
\section{Beyond Logical Correctness}
\subsection{Paraconsistent Reasoning}
In this section, we look at the GL $p(\Gamma,M;\mu\to 1)$ on the assumption that every model is possible, i.e., $0\notin p(M)$. To explain why $\mu$ approaching one, implemented by the limit operator $\lim_{\mu\to1}$, needs to be introduced to the GL, let us suppose $\alpha,\beta\in\Gamma$ such that $\ms{\beta}=\emptyset$. We then have
\begin{align*}
p(\alpha|\beta)=\frac{\displaystyle{\sum_{m}p(\alpha|m)p(\beta|m)p(m)}}{\displaystyle{\sum_{m}p(\beta|m)p(m)}}=\frac{\displaystyle{\sum_{m}p(\alpha|m)(1-\mu)p(m)}}{\displaystyle{\sum_{m}(1-\mu)p(m)}}.
\end{align*}
As discussed before, we must assume $\mu=1$ if we conform to the interpretation of formal logic. However, as seen in the previous sections, this causes a probability undefined due to division by zero. Given $\mu\neq 1$, however, we have
\begin{align*}
p(\alpha|\beta)=\frac{\displaystyle{\sum_{m}p(\alpha|m)(1-\mu)p(m)}}{\displaystyle{\sum_{m}(1-\mu)p(m)}}=\frac{\displaystyle{\sum_{m}p(\alpha|m)p(m)}}{\displaystyle{\sum_{m}p(m)}}=p(\alpha).
\end{align*}
Amongst $\mu\neq 1$, $\mu\to1$ is the only choice in terms of formal logic because, as shown below, only $\mu\to1$ and $\mu=1$ result in the same $p(\alpha)$.
\begin{align*}
p(\alpha)=\sum_{m}p(\alpha|m)p(m)=\sum_{m}\mu^{\ms{\alpha}_{m}}(1-\mu)^{1-\ms{\alpha}_{m}}p(m).
\end{align*}
\begin{example}
Let us see how limits work in practice. Consider the three conditional probabilities given different inconsistent premises shown on the right in Figure \ref{fig:ex_limit}. Given the probability distribution over models built with two symbols $r$ (meaning `rain') and $w$ (`wet') shown on the left, the conditional probability shown on the top right is expanded as follows.
\begin{figure}[t]
\begin{tabular}{cc}
 \begin{minipage}{0.35\hsize}
\begin{center}
{\small
\begin{tabular}{c|cc|c}
 & $rain$ & $wet$ & $p(M)$\\\hline
$m_{1}$ & 0 & 0 & 0.4\\
$m_{2}$ & 0 & 1 & 0.2\\
$m_{3}$ & 1 & 0 & 0.1\\
$m_{4}$ & 1 & 1 & 0.3
\end{tabular}
}
\end{center}
 \end{minipage}
 \begin{minipage}{0.64\hsize}
 \begin{center}
\includegraphics[scale=0.35]{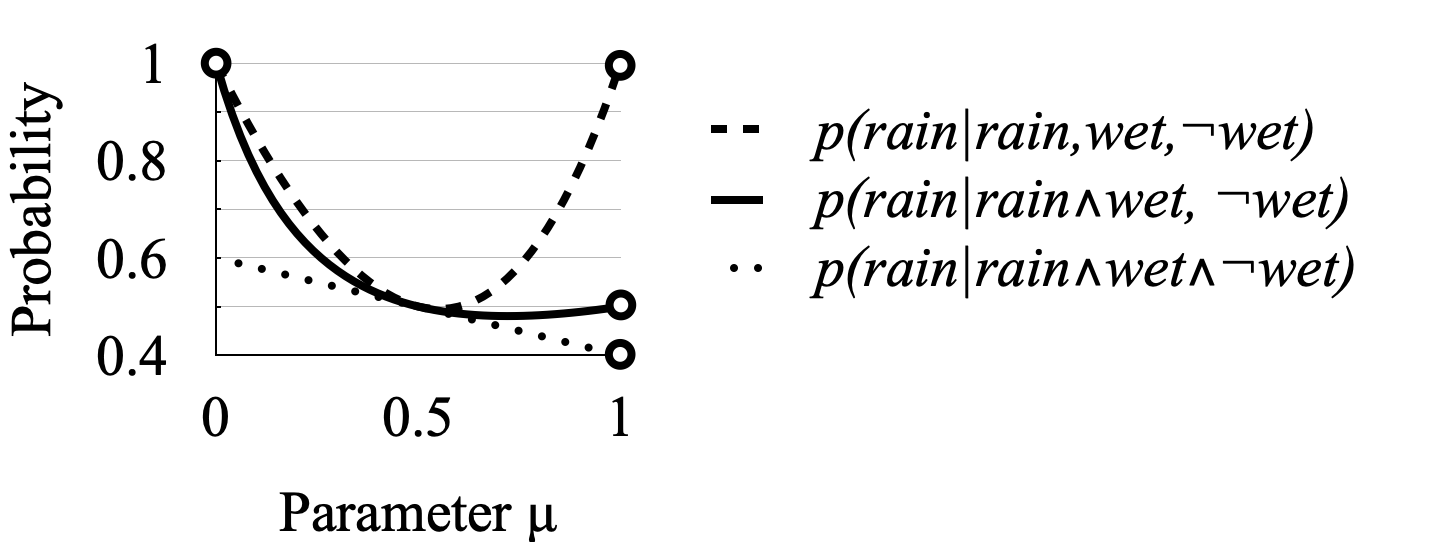}
 \end{center}
 \end{minipage}
\end{tabular}
\caption{The table on the left shows a model distribution and the graph on the right shows three examples of reasoning from inconsistency.}
\label{fig:ex_limit}
\end{figure}
\begin{align*}
p(r|r,w,\lnot w)&=\frac{\sum_{m}p(r|m)^{2}p(w|m)p(\lnot w|m)p(m)}{\sum_{m}p(r|m)p(w|m)p(\lnot w|m)p(m)}\\
&=\frac{(p(m_{1})+p(m_{2}))\mu(1-\mu)^{3}+(p(m_{3})+p(m_{4}))\mu^{3}(1-\mu)}{(p(m_{1})+p(m_{2}))\mu(1-\mu)^{2}+(p(m_{3})+p(m_{4})\mu^{2}(1-\mu)}\\
&=\frac{0.6\mu(1-\mu)^{3}+0.4\mu^{3}(1-\mu)}{0.6\mu(1-\mu)^{2}+0.4\mu^{2}(1-\mu)}
\end{align*}
The graph shown on the right in Figure \ref{fig:ex_limit} shows $p(rain|rain,wet,\lnot wet)$ given different $\mu$ values. The graph also includes the other two conditional probabilities calculated in the same manner. Each of the open circles represents an undefined value. This means that no substitution gives a probability, even though the curve approaches a certain probability. The certain probability can only be obtained by the use of limits. Indeed, given $\mu\to1$, the three conditional probabilities turn out to be 1, 0.5 and 0.4, respectively.
\end{example}
\par
We use maximal consistent sets to characterise the GL.
\begin{definition}[Maximal consistent subsets]
Let $S,\Delta\subseteq L$. $S\subseteq\Delta$ is a maximal consistent subset of $\Delta$ if $\ms{S}\neq\emptyset$ and $\ms{S\cup\{\alpha\}}=\emptyset$, for all $\alpha\in\Delta\setminus S$.
\end{definition}
We refer to a maximal consistent subset as a cardinality-maximal consistent subset when the set has the maximum cardinality. We use symbol $MCS(\Delta)$ to denote the set of the cardinality-maximal consistent subsets of $\Delta\subseteq L$. We use symbol $\ams{\Delta}$ to denote the set of the models of the cardinality-maximal consistent subsets of $\Delta$. In short, $\ams{\Delta}=\bigcup_{S\in MCS(\Delta)}\ms{S}$. Obviously, $\ams{\Delta}=\ms{\Delta}$ if there is a model of $\Delta$, i.e., $\ms{\Delta}\neq\emptyset$.
\begin{example}[Cardinality-maximal consistent sets]\label{ex:MCS}
Let $\Delta=\{$ $rain$, $wet$, $rain\to wet$, $\lnot wet\}$. There exist the following three maximal consistent subsets of $\Delta$.
\begin{itemize}
\item $S_{1}=\{rain,wet,rain\to wet\}$
\item $S_{2}=\{rain,\lnot wet\}$
\item $S_{3}=\{rain\to wet,\lnot wet\}$
\end{itemize}
Only $S_{1}$ is the cardinality-maximal consistent subset of $\Delta$, i.e., $MCS(\Delta)=\{S_{1}\}$. Therefore, $\ams{\Delta}=\bigcup_{S\in MCS(\Delta)}\ms{S}=\ms{S_{1}}=\{m_{4}\}$.
\end{example}
The following theorem relates the probability of a formula to the probability of the models of cardinality-maximal consistent sets.
\begin{theorem}\label{thrm:paraconsistent}
Let $p(\Gamma,M;\mu\to1)$ be a generative logic such that $0\notin p(M)$, $\alpha\in\Gamma$ and $\Delta\subseteq\Gamma$.
\begin{align*}
p(\alpha|\Delta)=
\begin{cases}
\displaystyle{\frac{\sum_{m\in\ams{\Delta}\cap\ms{\alpha}}p(m)}{\sum_{m\in\ams{\Delta}}p(m)}}&\text{if }\ams{\Delta}\neq\emptyset\\
p(\alpha)&\text{otherwise}
\end{cases}
\end{align*}
\end{theorem}
\begin{proof}
We use notation $|\Delta|_{m}$ to denote the number of formulas in $\Delta$ that are true in $m$, i.e., $|\Delta|_{m}=\sum_{\beta\in\Delta}\llbracket\beta\rrbracket_{m}$. Dividing models into $(\!(\Delta)\!)$ and the others, we have 
\begin{align*}
p(\alpha|\Delta)&=\lim_{\mu\rightarrow 1}\frac{\sum_{m}p(\alpha|m)p(m)p(\Delta|m)}{\sum_{m}p(m)p(\Delta|m)}\\
&=\lim_{\mu\rightarrow 1}\frac{\displaystyle{\sum_{\hat{m}\in(\!(\Delta)\!)}p(\alpha|\hat{m})p(\hat{m})p(\Delta|\hat{m})+\sum_{m\notin(\!(\Delta)\!)}p(\alpha|m)p(m)p(\Delta|m)}}{\displaystyle{\sum_{\hat{m}\in(\!(\Delta)\!)}p(\hat{m})p(\Delta|\hat{m})+\sum_{m\notin(\!(\Delta)\!)}p(m)p(\Delta|m)}}.
\end{align*}
Now, $p(\Delta|m)$ can be developed as follows, for all $m$ (regardless of the membership of $(\!(\Delta)\!)$).
\begin{align*}
p(\Delta|m)&=\prod_{\beta\in\Delta}p(\beta|m)=\prod_{\beta\in\Delta}\mu^{\llbracket\beta\rrbracket_{m}}(1-\mu)^{1-\llbracket\beta\rrbracket_{m}}\\
&=\mu^{\sum_{\beta\in\Delta}\llbracket\beta\rrbracket_{m}}(1-\mu)^{\sum_{\beta\in\Delta}(1-\llbracket\beta\rrbracket_{m})}=\mu^{|\Delta|_{m}}(1-\mu)^{|\Delta|-|\Delta|_{m}}
\end{align*}
Therefore, $p(\alpha|\Delta)=\lim_{\mu\rightarrow 1}\frac{W+X}{Y+Z}$ where
\begin{align*}
W&=\sum_{\hat{m}\in(\!(\Delta)\!)}p(\alpha|\hat{m})p(\hat{m})\mu^{|\Delta|_{\hat{m}}}(1-\mu)^{|\Delta|-|\Delta|_{\hat{m}}}\\
X&=\sum_{m\notin(\!(\Delta)\!)}p(\alpha|m)p(m)\mu^{|\Delta|_{m}}(1-\mu)^{|\Delta|-|\Delta|_{m}}\\
Y&=\sum_{\hat{m}\in(\!(\Delta)\!)}p(\hat{m})\mu^{|\Delta|_{\hat{m}}}(1-\mu)^{|\Delta|-|\Delta|_{\hat{m}}}\\
Z&=\sum_{m\notin(\!(\Delta)\!)}p(m)\mu^{|\Delta|_{m}}(1-\mu)^{|\Delta|-|\Delta|_{m}}.
\end{align*}
(Case: $\ams{\Delta}\neq\emptyset$) Since $\hat{m}\in \ams{\Delta}$ is a model of a cardinality-maximal consistent subset of $\Delta$, $|\Delta|_{\hat{m}}$ has the same value, for all $\hat{m}\in(\!(\Delta)\!)$. Therefore, the fraction can be simplified by dividing the denominator and numerator by $(1-\mu)^{|\Delta|-|\Delta|_{\hat{m}}}$. We thus have $p(\alpha|\Delta)=\lim_{\mu\rightarrow 1}\frac{W'+X'}{Y'+Z'}$ where
\begin{align*}
W'&=\sum_{\hat{m}\in(\!(\Delta)\!)}p(\alpha|\hat{m})p(\hat{m})\mu^{|\Delta|_{\hat{m}}}\\
X'&=\sum_{m\notin(\!(\Delta)\!)}p(\alpha|m)p(m)\mu^{|\Delta|_{m}}(1-\mu)^{|\Delta|_{\hat{m}}-|\Delta|_{m}}\\
Y'&=\sum_{\hat{m}\in(\!(\Delta)\!)}p(\hat{m})\mu^{|\Delta|_{\hat{m}}}\\
Z'&=\sum_{m\notin(\!(\Delta)\!)}p(m)\mu^{|\Delta|_{m}}(1-\mu)^{|\Delta|_{\hat{m}}-|\Delta|_{m}}.
\end{align*}
Applying the limit operation, we can cancel out $X'$ and $Z'$ and have
\begin{align*}
p(\alpha|\Delta)=\frac{\displaystyle{\sum_{\hat{m}\in(\!(\Delta)\!)}p(\alpha|\hat{m})p(\hat{m})}}{\displaystyle{\sum_{\hat{m}\in(\!(\Delta)\!)}p(\hat{m})}}=\frac{\displaystyle{\sum_{\hat{m}\in(\!(\Delta)\!)}1^{\llbracket\alpha\rrbracket_{\hat{m}}}0^{1-\llbracket\alpha\rrbracket_{\hat{m}}}p(\hat{m})}}{\displaystyle{\sum_{\hat{m}\in(\!(\Delta)\!)}p(\hat{m})}}.
\end{align*}
Since $1^{\llbracket\alpha\rrbracket_{\hat{m}}}0^{1-\llbracket\alpha\rrbracket_{\hat{m}}}=1^{1}0^{0}=1$ if $\hat{m}\in\llbracket\alpha\rrbracket$ and $1^{\llbracket\alpha\rrbracket_{\hat{m}}}0^{1-\llbracket\alpha\rrbracket_{\hat{m}}}=1^{0}0^{1}=0$ if $\hat{m}\notin\llbracket\alpha\rrbracket$, we have
\begin{align}\label{AIJ23:eq:paraconsistency}
p(\alpha|\Delta)=\frac{\sum_{\hat{m}\in(\!(\Delta)\!)\cap\llbracket\alpha\rrbracket}p(\hat{m})}{\sum_{\hat{m}\in(\!(\Delta)\!)}p(\hat{m})}.
\end{align}
(Case: $\ams{\Delta}=\emptyset$) Since $W=Y=0$, the fraction can be simplified as follows.
\begin{align*}
p(\alpha|\Delta)=\lim_{\mu\to 1}\frac{\displaystyle{\sum_{m}p(\alpha|m)p(m)\mu^{0}(1-\mu)^{|\Delta|}}}{\displaystyle{\sum_{m}p(m)\mu^{0}(1-\mu)^{|\Delta|}}}=\lim_{\mu\to 1}\displaystyle{\sum_{m}p(\alpha|m)p(m)}=p(\alpha).
\end{align*}
\end{proof}
\par
The following Corollary states that uncertain reasoning with the GL over symbols from inconsistency is a refined generalisation of the classical consequence relation.
\begin{corollary}\label{cor:paraconsistent_reasoning}
Let $p(\Gamma,M;\mu\to1)$ be a generative logic such that $0\notin p(M)$, $\alpha\in\Gamma$ and $\Delta\subseteq\Gamma$ such that $\ams{\Delta}\neq\emptyset$, $p(\alpha|\Delta)=1$ iff $S\cent\alpha$, for all cardinality-maximal consistent subsets $S$ of $\Delta$.
\end{corollary}
\begin{proof}
From Equation (\ref{AIJ23:eq:paraconsistency}), $p(\alpha|\Delta)=1$ holds iff $\ms{\alpha}\supseteq\ams{\Delta}$. By definition, $\ams{\Delta}=\bigcup_{S\in MCS(\Delta)}\ms{S}$ where $MCS(\Delta)$ is the set of the cardinality-maximal consistent subsets of $\Delta$. Therefore, $p(\alpha|\Delta)=1$ iff $\ms{\alpha}\supseteq\bigcup_{S\in MCS(\Delta)}\ms{S}$. In other words, $\ms{\alpha}\supseteq\ms{S}$, for all cardinality-maximal consistent subsets $S$ of $\Delta$, i.e., $S\models\alpha$.
\end{proof}
\par
Next, we examine the abstract inferential properties of the GL to establish that it acts as paraconsistent reasoning. Let $\alpha,\beta\in L$, $\Delta\subseteq L$ and $\vdash$ be a consequence relation over $L$, i.e., $\vdash\subseteq Pow(L)\times L$, where $Pow(L)$ denotes the power set of the language $L$. Logic $(L,\vdash)$ is said to be non-contradictory, non-trivial and explosive if it satisfies the following respective principles.
\begin{itemize}
\item Non-contradiction: $\exists\Delta\forall\alpha(\Delta\not\vdash\alpha$ or $\Delta\not\vdash\lnot\alpha)$
\item Non-triviality: $\exists\Delta\exists\alpha(\Delta\not\vdash\alpha)$
\item Explosion: $\forall\Delta\forall\alpha\forall\beta(\Delta,\alpha,\lnot\alpha\vdash\beta)$
\end{itemize}
A logic is paraconsistent iff it is not explosive and is sometimes called dialectical if it is contradictory \cite{Carnielli:07}. Now, let $\vsim_{\theta}$ be a consequence relation, referred to as a generative consequence, defined with a GL such that $\Delta\vsim_{\theta}\alpha$ iff $p(\alpha|\Delta)\geq\theta$, where $\theta$ is strictly larger than 0.5, i.e., $\theta\in(0.5,1]$. The following theorem states that the generative consequence defined with the GL with $\mu\to 1$ is paraconsistent but not dialectical, regardless of the presence of the assumption $0\notin p(M)$.
\begin{theorem}\label{thrm:6}
Let $p(\Gamma,M;\mu\to1)$ be a generative logic, $\theta\in(0.5,1]$ and $\vsim_{\theta}$ be a generative consequence defined with the generative logic. $\vsim_{\theta}$ satisfies the principles of non-contradiction and non-triviality but does not satisfy the principle of explosion.
\end{theorem}
\begin{proof}
(Non-contradiction) This is equivalent to $\exists\Delta\nexists\alpha(\Delta\vdash\alpha$ and $\Delta\vdash\lnot\alpha)$. Given $\Delta=\emptyset$, we show $\nexists\alpha(\vsim_{\theta}\alpha$ and $\vsim_{\theta}\lnot\alpha)$, for all $\theta\in(0.5,1]$. From definition, we can show there is no $\alpha$ such that $p(\alpha)\geq\theta$ and $p(\lnot\alpha)\geq\theta$. We have
\begin{align*}
p(\alpha)&=\sum_{m}\lim_{\mu\rightarrow 1}p(\alpha|m)p(m)=\sum_{m}\lim_{\mu\rightarrow 1}\mu^{\llbracket\alpha\rrbracket_{m}}(1-\mu)^{1-\llbracket\alpha\rrbracket_{m}}p(m)\\
&=\sum_{m}1^{\llbracket\alpha\rrbracket_{m}}0^{1-\llbracket\alpha\rrbracket_{m}}p(m)=\sum_{m\in\llbracket\alpha\rrbracket}p(m).
\end{align*}
In the second line, we used $1^{\llbracket\alpha\rrbracket_{m}}0^{1-\llbracket\alpha\rrbracket_{m}}=1^{1}0^{0}=1$ if $m\in\llbracket\alpha\rrbracket$ and $1^{\llbracket\alpha\rrbracket_{m}}0^{1-\llbracket\alpha\rrbracket_{m}}=1^{0}0^{1}=0$ if $m\notin\llbracket\alpha\rrbracket$. Similarly, we have
\begin{align*}
p(\lnot\alpha)&=\sum_{m}\lim_{\mu\rightarrow 1}p(\lnot\alpha|m)p(m)=\sum_{m}\lim_{\mu\rightarrow 1}\mu^{1-\llbracket\alpha\rrbracket_{m}}(1-\mu)^{\llbracket\alpha\rrbracket_{m}}p(m)\\
&=\sum_{m}1^{1-\llbracket\alpha\rrbracket_{m}}0^{\llbracket\alpha\rrbracket_{m}}p(m)=\sum_{m\notin\llbracket\alpha\rrbracket}p(m).
\end{align*}
In the second line, we used $1^{1-\llbracket\alpha\rrbracket_{m}}0^{\llbracket\alpha\rrbracket_{m}}=1^{1}0^{0}=1$ if $m\notin\llbracket\alpha\rrbracket$ and $1^{1-\llbracket\alpha\rrbracket_{m}}0^{\llbracket\alpha\rrbracket_{m}}=1^{0}0^{1}=0$ if $m\in\llbracket\alpha\rrbracket$. Now, $p(\alpha)+p(\lnot\alpha)=\sum_{m\in\llbracket\alpha\rrbracket} p(m)+\sum_{m\notin\llbracket\alpha\rrbracket} p(m)=\sum_{m}p(m)=1$. Therefore, there is no $\alpha$ such that $p(\alpha)>0.5$ and $p(\lnot\alpha)>0.5$.
\par
(Non-triviality) Let $\Delta=\emptyset$. Obviously, if $p(\alpha)\leq 0.5$ then $\not\vsim_{\theta}\alpha$, for all $\theta\in(0.5,1]$.
\par
(Explosion) Let $p(\beta)<0.5$. The following equation shows $p(\beta|\alpha,\lnot\alpha)<0.5$.
\begin{align*}
p(\beta|\alpha,\lnot\alpha)&=\lim_{\mu\rightarrow 1}\frac{\sum_{m}p(\alpha|m)p(\lnot\alpha|m)p(\beta|m)p(m)}{\sum_{m}p(\alpha|m)p(\lnot\alpha|m)p(m)}\\
&=\lim_{\mu\rightarrow 1}\frac{\mu(1-\mu)\sum_{m}p(\beta|v)p(m)}{\mu(1-\mu)\sum_{m}p(m)}=p(\beta)
\end{align*}
Namely, $\exists\alpha\exists\beta(\alpha,\lnot\alpha\not\vsim_{\theta}\beta)$, for all $\theta\in(0.5,1]$. This is a counter example.
\end{proof}
One cannot discuss whether the GL with $\mu=1$ acts as paraconsistent reasoning, or not. Indeed, the principle of explosion cannot be discussed because $\Delta,\alpha,\lnot\alpha\vsim_{\theta}\beta$ is undefined due to division by zero.
\par
Figure \ref{fig:paraconsistent_reasoning} illustrates paraconsistent reasoning with the GL $p(\Gamma,M,D;\mu\to1)$. It shows that only formulas constrain models. In contrast to consistent reasoning shown in Figure \ref{fig:consistent_reasoning}, however, only cardinality-maximal consistent sets constrain models as there is no model where all formulas in $\Delta$ are true.
\par
\begin{figure}[t]
\begin{center}
 \includegraphics[scale=0.4]{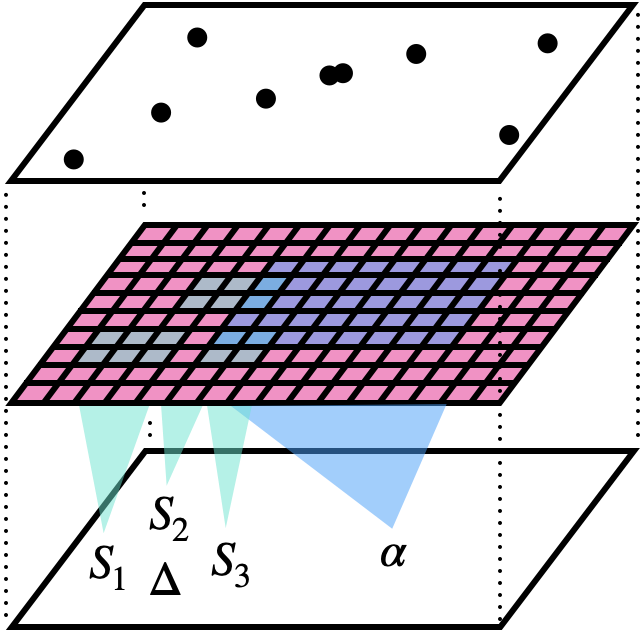}
  \caption{Paraconsistent reasoning with the GL $p(\Gamma,M,D;\mu\to1)$. $S_{1}$, $S_{2}$ and $S_{3}$ are the cardinality-maximal consistent subsets of $\Delta$.}
  \label{fig:paraconsistent_reasoning}
  \end{center}
\end{figure}
\subsection{Parapossible Reasoning}
In the previous section, we looked at the GL $p(\Gamma,M;\mu\to1)$ and characterised its certain reasoning as the classical entailment with cardinality-maximal consistent sets. The assumption we made for the characterisation is $0\notin p(M)$ meaning that every model has a non-zero probability. In this section, we look at the same GL without the assumption. This section thus aims to fully generalise the discussion of the previous section.
\par
We introduce maximal possible sets for a logical characterisation of the GL.
\begin{definition}[Maximal possible sets]
Let $S,\Delta\subseteq L$. $S\subseteq\Delta$ is a maximal possible subset of $\Delta$ if $\pms{S}\neq\emptyset$ and $\pms{S\cup\{\alpha\}}=\emptyset$, for all $\alpha\in\Delta\setminus S$.
\end{definition}
We refer to a maximal possible subset as a cardinality-maximal possible subset when the set has the maximum cardinality. We use symbol $MPS(\Delta)$ to denote the set of the cardinality-maximal possible subsets of $\Delta\subseteq L$. We use symbol $\pams{\Delta}$ to denote the set of possible models of the cardinality-maximal possible subsets of $\Delta$. In short, $\pams{\Delta}=\bigcup_{S\in MPS(\Delta)}\pms{S}$. Obviously, $\pams{\Delta}=\pms{\Delta}$ if there is a possible model of $\Delta$, i.e., $\pms{\Delta}\neq\emptyset$.
\begin{example}[Cardinality-maximal possible sets]\label{ex:MPS}
Suppose the probability distribution $p(M)=(m_{1},m_{2},m_{3},m_{4})=(0.9,0.1,0,0)$ in Figure \ref{fig:ex_limit}. Consider $\Delta=\{$ $rain$, $wet$, $rain\to wet,\lnot wet\}$. There are the following two maximal possible subsets of $\Delta$.
\begin{itemize}
\item $S_{1}=\{wet,rain\to wet\}$
\item $S_{2}=\{rain\to wet,\lnot wet\}$
\end{itemize}
Both $S_{1}$ and $S_{2}$ are the cardinality-maximal possible subsets of $\Delta$, i.e., $MPS(\Delta)=\{S_{1},S_{2}\}$. Only $m_{2}$ is the possible model of $S_{1}$ and $m_{1}$ is the possible model of $S_{2}$. Namely, $\pms{S_{1}}=\{m_{2}\}$ and $\pms{S_{2}}=\{m_{1}\}$. Therefore, $\pams{\Delta}=\bigcup_{S\in MPS(\Delta)}\pms{S}=\{m_{1},m_{2}\}$.
\end{example}
The following theorem relates the probability of a formula to the probability of its possible models.
\begin{theorem}\label{thrm:parapossible}
Let $p(\Gamma,M;\mu\to1)$ be a generative logic, $\alpha\in\Gamma$ and $\Delta\subseteq\Gamma$.
\begin{eqnarray*}
p(\alpha|\Delta)=
\begin{cases}
\displaystyle{\frac{\sum_{m\in\pams{\Delta}\cap\pms{\alpha}}p(m)}{\sum_{m\in\pams{\Delta}}p(m)}}&\text{if }\pams{\Delta}\neq\emptyset\\
p(\alpha)&\text{otherwise}
\end{cases}
\end{eqnarray*}
\end{theorem}
\begin{proof}
We use symbol $|\Delta|$ to denote the number of formulas in $\Delta$ and symbol $|\Delta|_{m}$ to denote the number of formulas in $\Delta$ that are true in $m$, i.e., $|\Delta|_{m}=\sum_{\beta\in\Delta}\m{\beta}_{m}$. Dividing models into $\pams{\Delta}$ and the others, we have
\begin{align*}
p(\alpha|\Delta)&=\lim_{\mu\rightarrow 1}\frac{\sum_{m}p(\alpha|m)p(m)p(\Delta|m)}{\sum_{m}p(m)p(\Delta|m)}\\
&=\lim_{\mu\rightarrow 1}\frac{\displaystyle{\sum_{\hat{m}\in\pams{\Delta}}p(\alpha|\hat{m})p(\hat{m})p(\Delta|\hat{m})+\sum_{m\notin\pams{\Delta}}p(\alpha|m)p(m)p(\Delta|m)}}{\displaystyle{\sum_{\hat{m}\in\pams{\Delta}}p(\hat{m})p(\Delta|\hat{m})+\sum_{m\notin\pams{\Delta}}p(m)p(\Delta|m)}}.
\end{align*}
Now, $p(\Delta|m)$ can be developed as follows, for all $m$.
\begin{align*}
p(\Delta|m)&=\prod_{\beta\in\Delta}p(\beta|m)=\prod_{\beta\in\Delta}\mu^{\m{\beta}_{m}}(1-\mu)^{1-\m{\beta}_{m}}\\
&=\mu^{\sum_{\beta\in\Delta}\m{\beta}_{m}}(1-\mu)^{\sum_{\beta\in\Delta}(1-\m{\beta}_{m})}=\mu^{|\Delta|_{m}}(1-\mu)^{|\Delta|-|\Delta|_{m}}
\end{align*}
Therefore, $p(\alpha|\Delta)=\lim_{\mu\rightarrow 1}\frac{W+X}{Y+Z}$ where
\begin{align*}
W&=\sum_{\hat{m}\in\pams{\Delta}}p(\alpha|\hat{m})p(\hat{m})\mu^{|\Delta|_{\hat{m}}}(1-\mu)^{|\Delta|-|\Delta|_{\hat{m}}}\\
X&=\sum_{m\notin\pams{\Delta}}p(\alpha|m)p(m)\mu^{|\Delta|_{m}}(1-\mu)^{|\Delta|-|\Delta|_{m}}\\
Y&=\sum_{\hat{m}\in\pams{\Delta}}p(\hat{m})\mu^{|\Delta|_{\hat{m}}}(1-\mu)^{|\Delta|-|\Delta|_{\hat{m}}}\\
Z&=\sum_{m\notin\pams{\Delta}}p(m)\mu^{|\Delta|_{m}}(1-\mu)^{|\Delta|-|\Delta|_{m}}.
\end{align*}
(Case: $\pams{\Delta}\neq\emptyset$) Now, for all $m$, if $m\notin\pams{\Delta}$ then $m$ is impossible or $m$ is a possible model of a subset of $\Delta$ that is not a cardinality-maximal possible subset of $\Delta$. Therefore, $p(m)=0$ or there is $\hat{m}\in\pams{\Delta}$ such that $|\Delta|_{m}<|\Delta|_{\hat{m}}$. $|\Delta|_{\hat{m}_{1}}=|\Delta|_{\hat{m}_{2}}$ by definition, for all $\hat{m}_{1},\hat{m}_{2}\in\pams{\Delta}$. The fraction thus can be simplified by dividing the denominator and numerator by $(1-\mu)^{|\Delta|-|\Delta|_{\hat{m}}}$. We thus have $p(\alpha|\Delta)=\lim_{\mu\rightarrow 1}\frac{W'+X'}{Y'+Z'}$ where
\begin{align*}
W'&=\sum_{\hat{m}\in\pams{\Delta}}p(\alpha|\hat{m})p(\hat{m})\mu^{|\Delta|_{\hat{m}}}\\
X'&=\sum_{m\notin\pams{\Delta}}p(\alpha|m)p(m)\mu^{|\Delta|_{m}}(1-\mu)^{|\Delta|_{\hat{m}}-|\Delta|_{m}}\\
Y'&=\sum_{\hat{m}\in\pams{\Delta}}p(\hat{m})\mu^{|\Delta|_{\hat{m}}}\\
Z'&=\sum_{m\notin\pams{\Delta}}p(m)\mu^{|\Delta|_{m}}(1-\mu)^{|\Delta|_{\hat{m}}-|\Delta|_{m}}.
\end{align*}
Applying the limit operation, we can cancel out $X'$ and $Z'$ and have
\begin{align*}
p(\alpha|\Delta)=\frac{\displaystyle{\sum_{\hat{m}\in\pams{\Delta}}p(\alpha|\hat{m})p(\hat{m})}}{\displaystyle{\sum_{\hat{m}\in\pams{\Delta}}p(\hat{m})}}=\frac{\displaystyle{\sum_{\hat{m}\in\pams{\Delta}}1^{\m{\alpha}_{\hat{m}}}0^{1-\m{\alpha}_{\hat{m}}}p(\hat{m})}}{\displaystyle{\sum_{\hat{m}\in\pams{\Delta}}p(\hat{m})}}.
\end{align*}
Since $1^{\ms{\alpha}_{\hat{m}}}0^{1-\ms{\alpha}_{\hat{m}}}=1^{1}0^{0}=1$ if $\hat{m}\in\ms{\alpha}$ and $1^{\ms{\alpha}_{\hat{m}}}0^{1-\ms{\alpha}_{\hat{m}}}=1^{0}0^{1}=0$ if $\hat{m}\notin\ms{\alpha}$, we have
\begin{align}\label{AIJ23:eq:impossibility}
p(\alpha|\Delta)=\frac{\sum_{\hat{m}\in\pams{\Delta}\cap\m{\alpha}}p(\hat{m})}{\sum_{\hat{m}\in\pams{\Delta}}p(\hat{m})}=\frac{\sum_{\hat{m}\in\pams{\Delta}\cap\pms{\alpha}}p(\hat{m})}{\sum_{\hat{m}\in\pams{\Delta}}p(\hat{m})}.
\end{align}
(Case: $\pams{\Delta}=\emptyset$) Since $W=Y=0$, the fraction can be simplified as follows.
\begin{align*}
p(\alpha|\Delta)=\lim_{\mu\to 1}\frac{\displaystyle{\sum_{m}p(\alpha|m)p(m)\mu^{0}(1-\mu)^{|\Delta|}}}{\displaystyle{\sum_{m}p(m)\mu^{0}(1-\mu)^{|\Delta|}}}=\lim_{\mu\to 1}\displaystyle{\sum_{m}p(\alpha|m)p(m)}=p(\alpha).
\end{align*}
\end{proof}
The following Corollary states that uncertain reasoning with the GL over symbols from impossibility is a refined generalisation of the alternative consequence relation. 
\begin{corollary}\label{cor:parapossible_reasoning}
Let $p(\Gamma,M;\mu\to1)$ be a generative logic, $\alpha\in\Gamma$ and $\Delta\subseteq\Gamma$ such that $\pams{\Delta}\neq\emptyset$. $p(\alpha|\Delta)=1$ iff $S\ent\alpha$, for all cardinality-maximal possible subsets $S$ of $\Delta$.
\end{corollary}
\begin{proof}
From Equation (\ref{AIJ23:eq:impossibility}), $p(\alpha|\Delta)=1$ iff $\pams{\Delta}\subseteq\pms{\alpha}$. Since $\pams{\Delta}=\bigcup_{S\in MPS(\Delta)}\pms{S}$, $p(\alpha|\Delta)=1$ iff $\bigcup_{S\in MPS(\Delta)}\pms{S}\subseteq\pms{\alpha}$. Therefore, $\pms{S}\subseteq\pms{\alpha}$, for all $S\in MPS(\Delta)$.
\end{proof}
\par
Figure \ref{fig:parapossible_reasoning} illustrates parapossible reasoning with the GL $p(\Gamma,M,D;\mu\to1)$. It shows that models are not only constrained by formulas but also restricted by data. In contrast to paraconsistent reasoning shown in Figure \ref{fig:paraconsistent_reasoning}, however, only cardinality-maximal possible sets constrain models as there is no possible model where all formulas in $\Delta$ are true.
\par
\begin{figure}[t]
\begin{center}
 \includegraphics[scale=0.4]{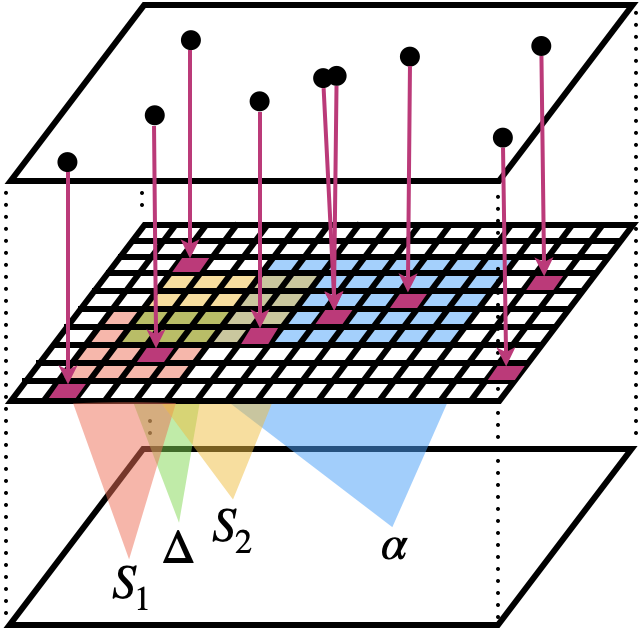}
  \caption{Parapossible reasoning with the GL $p(\Gamma,M,D;\mu\to1)$. $S_{1}$ and $S_{2}$ are the cardinality-maximal possible subsets of $\Delta$.}
  \label{fig:parapossible_reasoning}
  \end{center}
\end{figure}
\section{Machine-learning Correctness}
\subsection{Generative and Predictive Reasoning}
We have looked at the GL $p(\Gamma$, $M$, $D; \mu=1)$ and $p(\Gamma$, $M$, $D; \mu\to1)$ and characterised their symbolic reasoning using the classical and alternative consequence relations, respectively. We have seen that the GL is a theory of inference unifying uncertain reasoning from consistency, possibility, inconsistency and impossibility. In this section, we look at GL, especially $p(\Gamma$, $M$, $D; \mu\to1)$ and $p(\Gamma$, $M$, $D; \mu<1)$, from a machine learning perspective. The GL are applied to the MNIST dataset to empirically discuss how well they solve the machine learning problem. They are also compared to the K nearest neighbour method to theoretically discuss the link to the well-known machine learning approach.
\par
The MNIST dataset contains 70,000 images, 60,000 training and 10,000 test images, of handwritten digits from 0 to 9. Each image comprises 784 ($28\times 28$ in width and height) pixels, where each pixel has a greyscale from 0 (black) to 255 (white). Let $p(\Gamma$, $M$, $D; \mu)$ be a GL where $D$ is realised by the training images and $M$ by all possible $2^{28\times 28}$ sequences of 0 (or `false' meaning black) or 1 (or `true' meaning white). Here, we assume that each element of the sequences is 0 if the greyscale is below 30 and 1 otherwise. $\Gamma$ is realised by symbolic representations of any inquiries about test images. For example, $N_{i}\in\Gamma$ represents that an image displays the digit $i$, and $P_{j}\in\Gamma$ that the pixel $j$ of an image is white. The following proposition is useful in machine learning contexts to see symbolic reasoning from models as symbolic reasoning from data.
\begin{proposition}\label{prop:data-basedCP}
Let $p(\Gamma$, $M$, $D; \mu)$ be a generative logic, $\alpha\in \Gamma$ and $\Delta\subseteq \Gamma$.
\begin{align*}
p(\alpha|\Delta)=\sum_{m}p(\alpha|m)p(m|\Delta)=\sum_{d}p(\alpha|d)p(d|\Delta)
\end{align*}
\end{proposition}
\begin{proof}
See \ref{proof}.
\end{proof}
\par
Given the GL, we consider the two reasoning tasks: generation and prediction. The first task is to generate a standard image of a digit from the digit. We use the GL $p(\Gamma$, $M$, $D; \mu=1)$ as there is at least one image per digit in the training images (i.e., each digit is possible). The following conditional probability performs the generation task.
\begin{align}\label{eq:mnist2}
p(P_{j}|N_{i})=\sum_{d}p(P_{j}|d)p(d|N_{i})
\end{align}
The equation states that the pixel $j$ of a test image is predicted only using all the training images with the digit $i$. A standard image of the digit is generated by normalising $p(P_{j}|N_{i})\in[0,1]$ to the greyscale from 0 (black) to 255 (white), for all pixels $j$ ($1\leq j\leq 28\times 28$). The second task is to predict the digit from an image. Since no test image usually appears in the training images (i.e., the test images are all impossible), we use the GL $p(\Gamma$, $M$, $D; \mu\to1)$ and $p(\Gamma$, $M$, $D; \mu< 1)$, but not $p(\Gamma$, $M$, $D; \mu=1)$. The following conditional probability performs the prediction task.
\begin{align}\label{eq:mnist1}
p(N_{i}|P_{1},P_{2},...,P_{28\times 28})&=\sum_{d}p(N_{j}|d)p(d|P_{1},P_{2},...,P_{28\times 28})
\end{align}
The equation implies that, in contrast to $p(N_{i})=\sum_{d}p(N_{j}|d)p(d)$, predictions are based on the posterior distribution over the training images updated by the observation of the pixels of a test image. The generative and predictive reasoning tasks are illustrated in Figure \ref{fig:prediction_and_generation}.
\begin{figure}[t]
\begin{tabular}{cc}
\begin{minipage}{0.5\hsize}
\begin{center}
\includegraphics[scale=0.35]{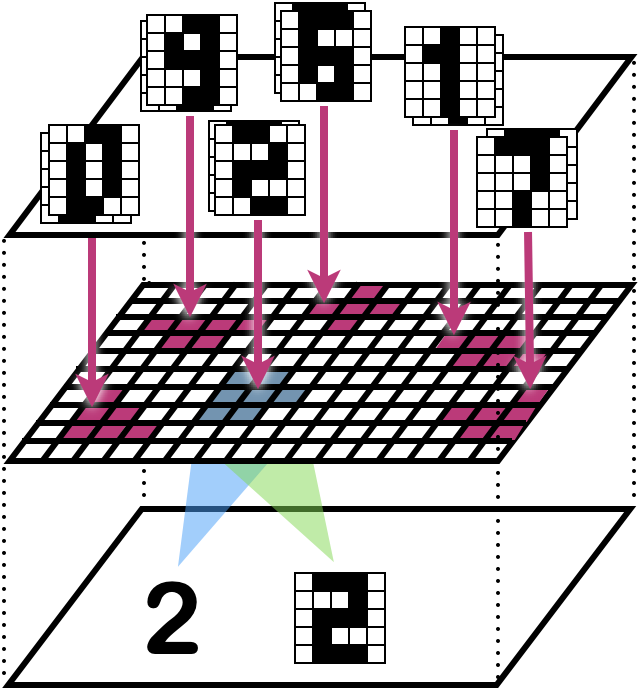}
\end{center}
\end{minipage}
\begin{minipage}{0.5\hsize}
\begin{center}
\includegraphics[scale=0.35]{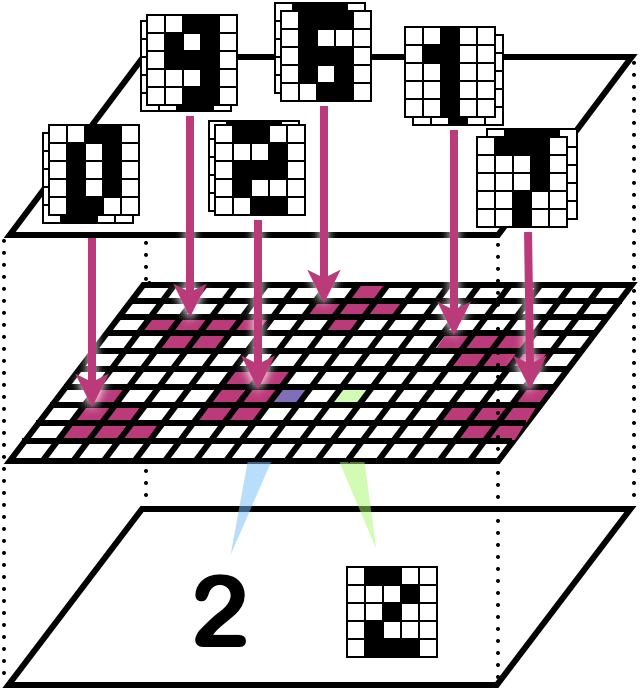}
\end{center}
\end{minipage}
\end{tabular}
\caption{The image on the left illustrates a generative reasoning task where the standard image (i.e., the pixels on the bottom layer) is inferred from the digit 2. The right-hand side illustrates a predictive reasoning task where the digit 2 is inferred from the new image.}
\label{fig:prediction_and_generation}
\end{figure}
\par
Figure \ref{fig:generation} shows standard digit images we generated using Equation (\ref{eq:mnist2}). The result is not only intuitively reasonable but also statistically justified by the following equation obtained by expanding Theorem \ref{thrm:consistent} for data-based reasoning.
\begin{align}\label{eq:generation}
p(P_{j}|N_{i})=\frac{\sum_{m}\ms{P_{j}}_{m}\ms{N_{i}}_{m}p(m)}{\sum_{m}\ms{N_{i}}_{m}p(m)}=\frac{\sum_{d}\ms{P_{j}}_{m(d)}\ms{N_{i}}_{m(d)}}{\sum_{d}\ms{N_{i}}_{m(d)}}
\end{align}
The equation states that the greyscale of each pixel of a standard image for a digit is based on the mean of the colours, black or white, of the pixel of all the images with the digit.
\par
Figure \ref{fig:LCs} shows the learning curves we obtained using Equation (\ref{eq:mnist1}). The result is still not only empirically reasonable but also theoretically justified. Indeed, the GL $p(\Gamma$, $M$, $D; \mu\to1)$ and $p(\Gamma$, $M$, $D; \mu<1)$ behave similarly to the K nearest neighbour (K-NN) method, which is a supervised machine-learning approach that classifies test examples by a majority vote from the K closest training examples. In particular, Equation (\ref{AIJ23:eq:impossibility}) for the GL $p(\Gamma$, $M$, $D; \mu\to1)$ can be expanded as follows for data-based reasoning, where $\bm{P}=\{P_{1},P_{2},...,P_{28\times 28}\}$.
\begin{align}\label{eq:prediction}
p(N_{i}|\bm{P})&=\frac{\sum_{m}\pams{\bm{P}}_{m}\pms{N_{i}}_{m}p(m)}{\sum_{m}\pams{\bm{P}}_{m}p(m)}=\frac{\sum_{d}\pams{\bm{P}}_{m(d)}\pms{N_{i}}_{m(d)}}{\sum_{d}\pams{\bm{P}}_{m(d)}}
\end{align}
The equation states that the predictive probability of a digit is the fraction of images with the digit from all the training images whose pixel values are closest to those of the test image. The denominator is the number of training images whose pixel values are closest to those of the test image. This is because those training images cause the models where the maximum number of pixel values of the test image is true. Amongst them, the numerator is the number of training images with the same digit. Thus, the conditional probability can be seen as a fully non-parametric all-nearest-neighbour method. This is a reasonable solution to a well-known problem that it is often difficult to settle an appropriate value of K for K-NN methods. It is noteworthy that the search for and the use of the closest neighbours are unified by the GL simply as probabilistic reasoning.
\begin{figure}[t]
\begin{center}
\includegraphics[scale=0.22]{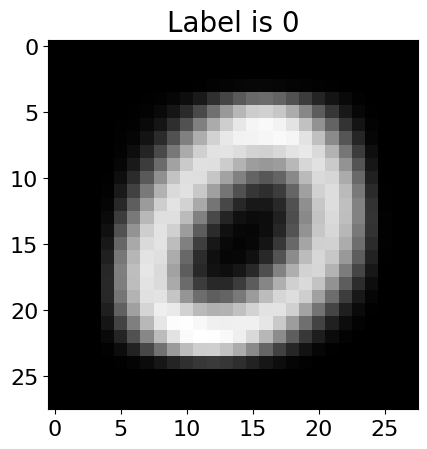}
\includegraphics[scale=0.22]{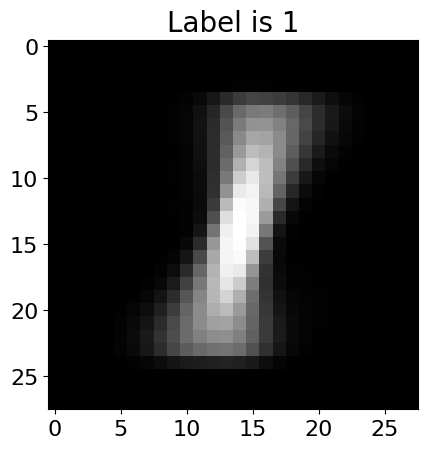}
\includegraphics[scale=0.22]{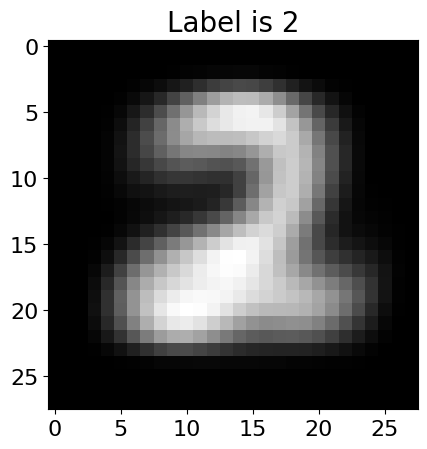}
\includegraphics[scale=0.22]{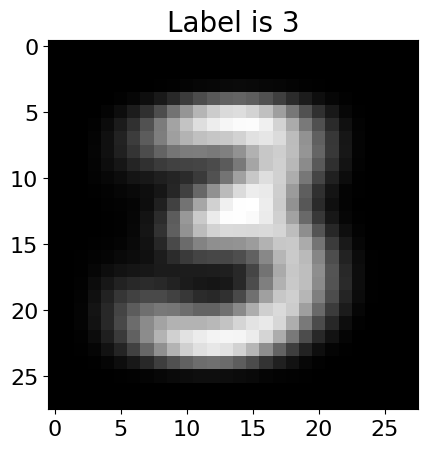}
\includegraphics[scale=0.22]{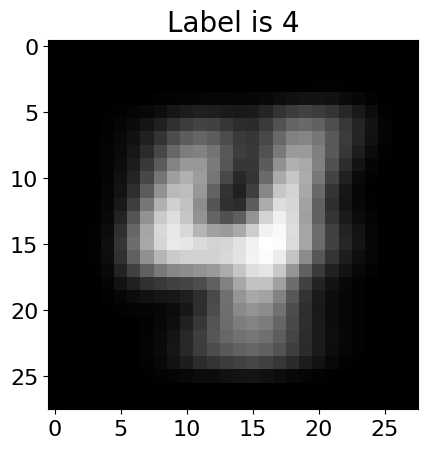}
\\
\includegraphics[scale=0.22]{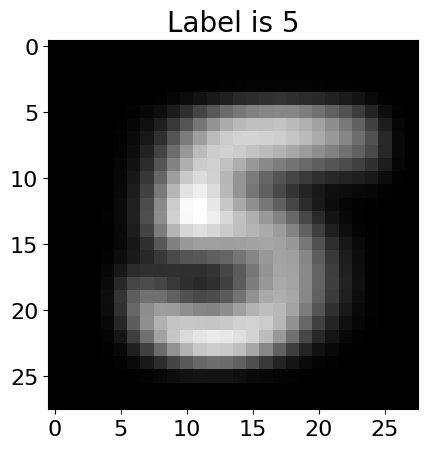}
\includegraphics[scale=0.22]{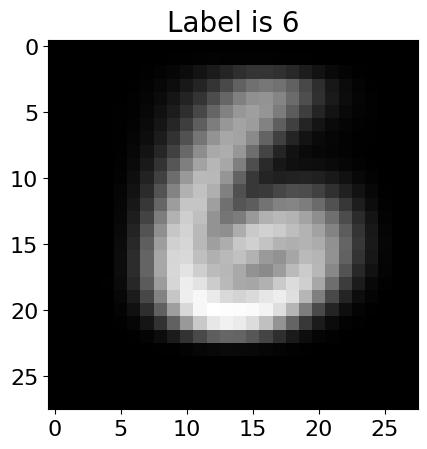}
\includegraphics[scale=0.22]{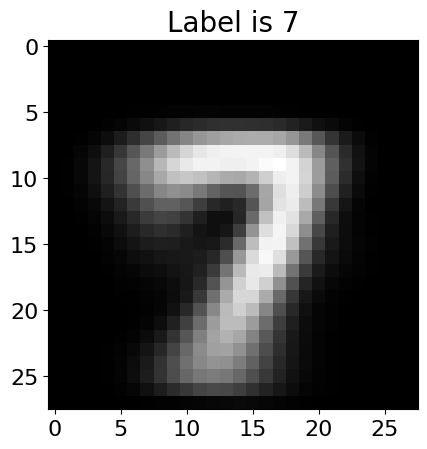}
\includegraphics[scale=0.22]{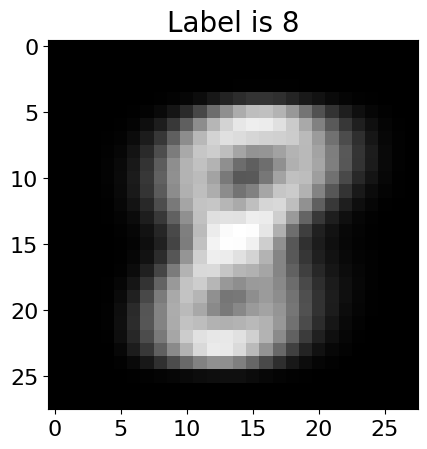}
\includegraphics[scale=0.22]{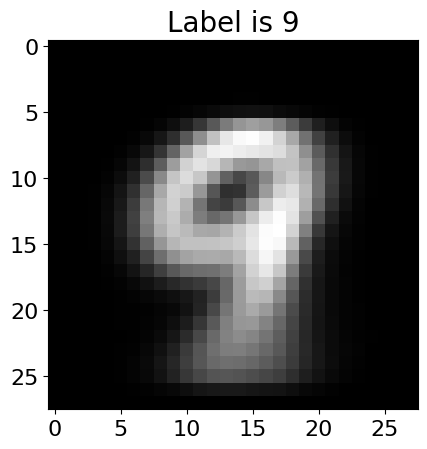}
 \end{center}
\caption{Standard images of all the digits generated by the GL $p(\Gamma,M,D;\mu=1)$, where all the 70,000 MNIST images realise $D$.}
\label{fig:generation}
\end{figure}
\par
The GL $p(\Gamma$, $M$, $D; \mu<1)$ can be seen as a smoothed version of the GL $p(\Gamma$, $M$, $D; \mu\to1)$. Figure \ref{fig:LCs} implies that the latter GL experiences overfitting like the 1-NN method, while the former GL succeeds in mitigating overfitting. The former GL performs better than the K-NN method particularly when the number of training images are small. The parameter $\mu$ of the former GL is also less sensitive than the parameter K of the K-NN method with respect to the learning performance. \ref{experiment} shows the ROC curves and their AUC per digit drawn using different parameter values.
\par
Let $\bm{P}_{test}$ be a sequence of the pixels of a test image and each of $\bm{P}_{train1}$ and $\bm{P}_{train2}$ be a sequence of the pixels of a training image. Suppose that $\bm{P}_{test}$ and $\bm{P}_{train1}$ have the same $z$ pixel values, and $\bm{P}_{test}$ and $\bm{P}_{train2}$ have the same $z-1$ pixel values. The probability that the training-image pixels generate the test-image pixels is given as follows.
\begin{align*}
p(\bm{P}^{test}|\bm{P}^{train_{1}})&=\prod_{i=1}^{28\times 28}p(P_{i}^{test}|m(\bm{P}^{train_{1}}))=\mu^{z}(1-\mu)^{28\times28-z}\\
p(\bm{P}^{test}|\bm{P}^{train_{2}})&=\prod_{i=1}^{28\times 28}p(P_{i}^{test}|m(\bm{P}^{train_{2}}))=\mu^{z-1}(1-\mu)^{28\times28-(z-1)}
\end{align*}
We thus have $p(\bm{P}^{test}|\bm{P}^{train_{1}})=\frac{\mu}{1-\mu}p(\bm{P}^{test}|\bm{P}^{train_{2}})$. For example, given $\mu=0.8$, $\bm{P}_{test}$ is four times more likely to be generated from $\bm{P}_{train1}$ than $\bm{P}_{train2}$. When the test image has an odd number of nearest test images, the GL with $\mu$ close to one behaves like a majority vote. When the test image has an even number of nearest training images, the GL additionally searchers for sub-nearest training images until an odd number of such images is found. This is a theoretical advantage of the GL $p(\Gamma$, $M$, $D; \mu<1)$ over the GL $p(\Gamma$, $M$, $D; \mu\to1)$ and the K-NN method. Again, the search for and the use of the nearest and sub-nearest training images are unified by the GL as probabilistic reasoning.
\begin{figure}[t]
\begin{center}
\includegraphics[scale=0.3]{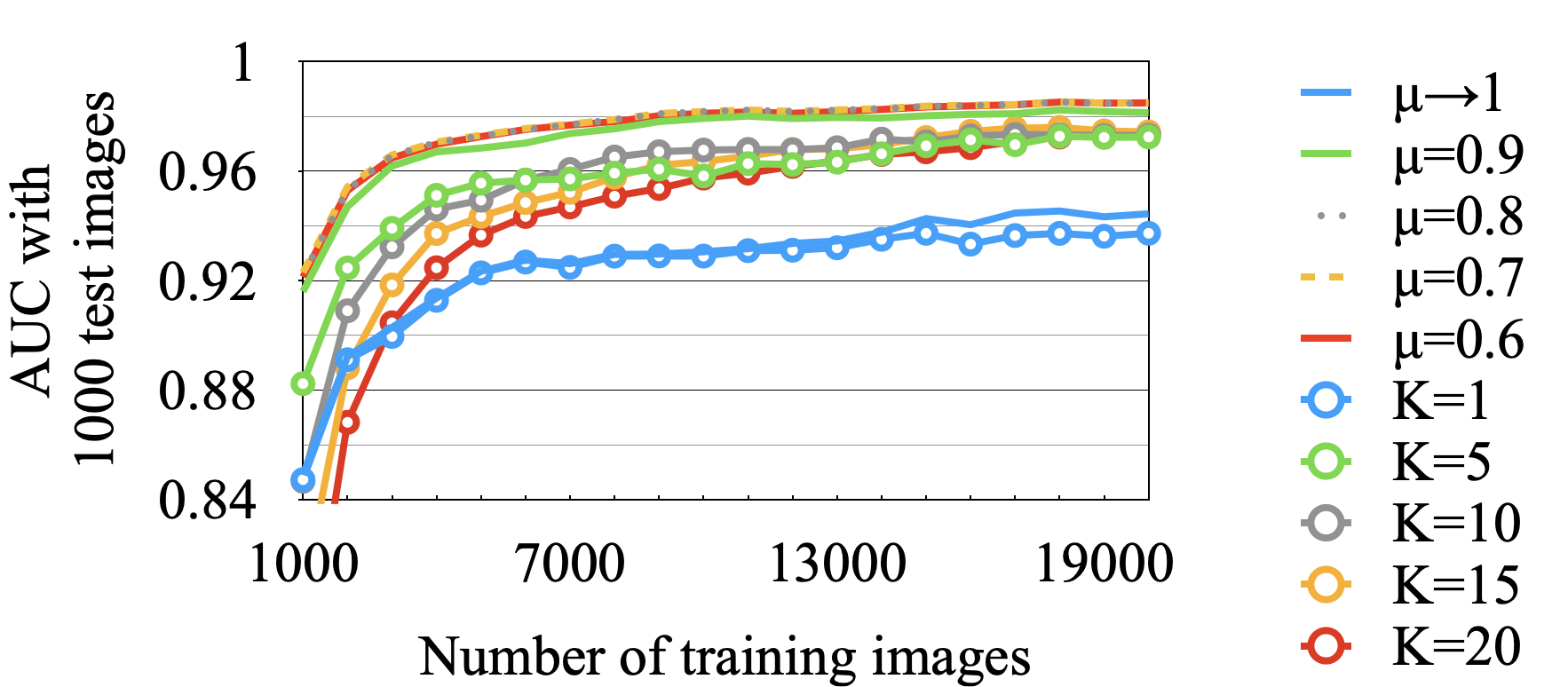}
\end{center}
\caption{The learning curves of the GL $p(\Gamma$, $M$, $D; \mu)$ and the K-NN method. The K-NN method was built using the `KNeighborsClassifier' function \cite{scikit-learn} with the `uniform' weights and `auto' algorithm. The y-axis is the AUC (area under the ROC (receiver operating characteristic) curves). The training and test images were both extracted from the beginning, from the 60,000 training and 10,000 test images, respectively.}
\label{fig:LCs}
\end{figure}
\section{Conclusions}\label{sec:conclusion}
We proposed a simple theory of inference to explain how some basic approaches to logical reasoning and statistical learning stem from a common principle. We simply modelled how data causes symbolic knowledge in terms of its satisfiability in formal logic. Symbolic reasoning emerged as a result of going the causality forwards and backwards. The forward and backward processes correspond to the interpretation and inverse interpretation in formal logic, respectively. We looked at several criteria for the statistical, logical and machine learning correctness of the theory. Amongst the criteria, the most important fact would be that maximal likelihood estimation, the classical and alternative consequence relations and the K nearest neighbour method all justify the statistical, logical and machine learning correctness of the theory, respectively. We indeed showed that probabilistic reasoning over models is consistent with maximal likelihood estimation, that probabilistic reasoning over symbols unifies reasoning from consistency, possibility, inconsistency and impossibility, and that the reasoning from impossibility can be seen as a fully non-parametric all-nearest neighbour method and its reasonable refinement for overfitting mitigation.
\appendix
\section{Graphical Abstract}\label{grab}
\begin{center}
\includegraphics[scale=0.27]{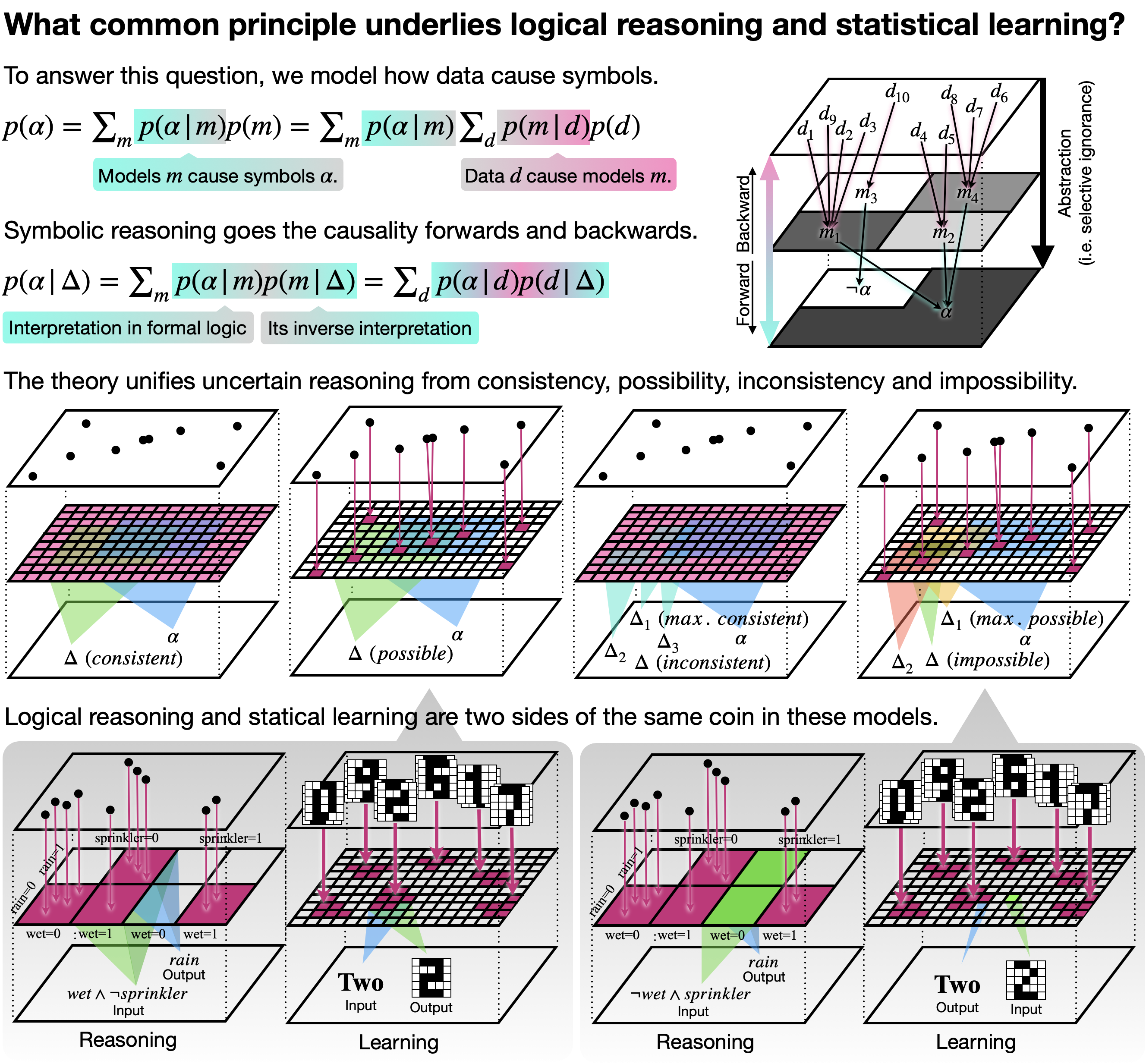}
\end{center}

\section{Proofs}\label{proof}
\begin{proof}[Proposition 1]
Let $\alpha,\beta\in L$. We need to show the following three properties.
\begin{enumerate}
\item $0\leq p(\alpha=i)$ holds, for all $i\in\{0,1\}$.
\item $\sum_{i\in\{0,1\}}p(\alpha=i)=1$ holds.
\item $p(\alpha\lor\beta=i)=p(\alpha=i)+p(\beta=i)-p(\alpha\land\beta=i)$ holds, for all $i\in\{0,1\}$.
\end{enumerate}
(1) $p(\alpha=i)=\sum_{m}p(\alpha=i|m)p(m)$. Both $p(\alpha=i|m)$ and $p(m)$ cannot be negative. (2) Since $\llbracket\alpha=0\rrbracket_{m}=1-\llbracket\alpha=1\rrbracket_{m}$, we have
\begin{align*}
p(\alpha=0|m)+p(\alpha=1|m)&=\mu^{\llbracket\alpha=0\rrbracket_{m}}(1-\mu)^{1-\llbracket\alpha=0\rrbracket_{m}}+\mu^{\llbracket\alpha=1\rrbracket_{m}}(1-\mu)^{1-\llbracket\alpha=1\rrbracket_{m}}\\
&=\mu^{1-\llbracket\alpha=1\rrbracket_{m}}(1-\mu)^{\llbracket\alpha=1\rrbracket_{m}}+\mu^{\llbracket\alpha=1\rrbracket_{m}}(1-\mu)^{1-\llbracket\alpha=1\rrbracket_{m}}.
\end{align*}
If $\llbracket\alpha=1\rrbracket_{m}=1$ then $p(\alpha=0|m)+p(\alpha=1|m)=(1-\mu)+\mu=1$. If $\llbracket\alpha=1\rrbracket_{m}=0$ then $p(\alpha=0|m)+p(\alpha=1|m)=\mu+(1-\mu)=1$. Therefore, we have
\begin{align*}
p(\alpha=0)+p(\alpha=1)&=\sum_{m}p(\alpha=0|m)p(m)+\sum_{m}p(\alpha=1|m)p(m)\\
&=\sum_{m}p(m)\{p(\alpha=0|m)+p(\alpha=1|m)\}=\sum_{m}p(m)=1.
\end{align*}
(3) From (2), it is sufficient to show only case $i=1$ because case $i=0$ can be developed as follows.
\begin{align*}
1-p(\alpha\lor\beta=1)=1-\{p(\alpha=1)+p(\beta=1)-p(\alpha\land\beta=1)\}
\end{align*}
It is sufficient to show $p(\alpha\lor\beta=1|m)=p(\alpha=1|m)+p(\beta=1|m)-p(\alpha\land\beta=1|m)$, for all $m$, since the following holds.
\begin{align*}
\sum_{m}p(\alpha\lor\beta=1|m)p(m)=\sum_{m}\{p(\alpha=1|m)+p(\beta=1|m)-p(\alpha\land\beta=1|m)\}p(m)
\end{align*}
By case analysis, the right expressions can have either of the following four cases.
\begin{eqnarray}
(1-\mu)+(1-\mu)-(1-\mu)&=&1-\mu\label{1}\\
(1-\mu)+\mu-(1-\mu)&=&\mu \label{2}\\
\mu +(1-\mu)-(1-\mu)&=&\mu\label{3}\\
\mu+\mu-\mu&=&\mu \label{4}
\end{eqnarray}
where (\ref{1}), (\ref{2}), (\ref{3}) and (\ref{4}) are obtained in the cases ($\llbracket\alpha=1\rrbracket_{m}=\llbracket\beta=1\rrbracket_{m}=0$), ($\llbracket\alpha=1\rrbracket_{m}=0$ and $\llbracket\beta=1\rrbracket_{m}=1$), ($\llbracket\alpha=1\rrbracket_{m}=1$ and $m\in\llbracket\beta=1\rrbracket_{m}=0$), and ($\llbracket\alpha=1\rrbracket_{m}=\llbracket\beta=1\rrbracket_{m}=1$), respectively.  All of the results are consistent with the left expression, i.e., $p(\alpha\lor\beta=1|m)$.
\end{proof}
\begin{proof}[Proposition 2]
For all $m$, there are only two cases: $p(\alpha=1|m)=\mu$ and $p(\alpha=1|m)=1-\mu$. $p(\alpha=1|m)=\mu$ iff $p(\lnot\alpha=1|m)=1-\mu$, and $p(\alpha=1|m)=1-\mu$ iff $p(\lnot\alpha=1|m)=\mu$. Therefore, $p(\alpha=1|m)=1-p(\lnot\alpha=1|m)$. From (2) of Proposition 1, we have 
\begin{align*}
p(\alpha=1)&=\sum_{m}p(\alpha=1|m)p(m)=\sum_{m}\{1-p(\lnot\alpha=1|m)\}p(m)\\
&=\sum_{m}p(\lnot\alpha=0|m)p(m)=p(\lnot\alpha=0).
\end{align*}
\end{proof}
\begin{proof}[Proposition 4]
Since $p(\Delta|m)=p(m|\Delta)p(\Delta)/p(m)$ due to Bayes' theorem, we have
\begin{align*}
p(\alpha|\Delta)&=\frac{p(\alpha,\Delta)}{p(\Delta)}=\frac{\sum_{m}p(\alpha|m)p(\Delta|m)p(m)}{\sum_{m}p(\Delta|m)p(m)}\\
&=\frac{\sum_{m}p(\alpha|m)p(m|\Delta)}{\sum_{m}p(m|\Delta)}=\sum_{m}p(\alpha|m)p(m|\Delta).
\end{align*}
The first expression of Proposition \ref{prop:data-basedCP} can be written as follows.
\begin{align*}
p(\alpha|\Delta)&=\frac{\sum_{d}p(d)\sum_{m}p(\alpha|m)p(\Delta|m)p(m|d)}{p(\Delta)}=\frac{\sum_{d}p(d)p(\alpha|m(d))p(\Delta|m(d))}{p(\Delta)}
\end{align*}
The two terms of the third expression of Proposition \ref{prop:data-basedCP} can be written as follows.
\begin{align*}
p(\alpha|d)&=\frac{\sum_{m}p(\alpha,d,m)}{p(d)}=\frac{\sum_{m}p(\alpha|m)p(m|d)p(d)}{p(d)}=p(\alpha|m(d))\\
p(d|\Delta)&=\frac{\sum_{m}p(d,\Delta,m)}{p(\Delta)}=\frac{\sum_{m}p(\Delta|m)p(m|d)p(d)}{p(\Delta)}=\frac{p(\Delta|m(d))p(d)}{p(\Delta)}
\end{align*}
Taking a summation over data, we have $\sum_{d}p(\alpha|d)p(d|\Delta)=p(\alpha|\Delta)$.
\end{proof}
\begin{figure}[t]
\section{Experiment}\label{experiment}
\begin{center}
\includegraphics[scale=0.19]{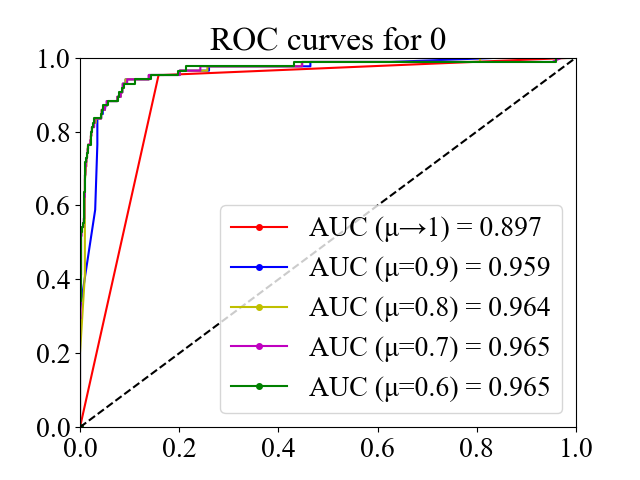}
\includegraphics[scale=0.19]{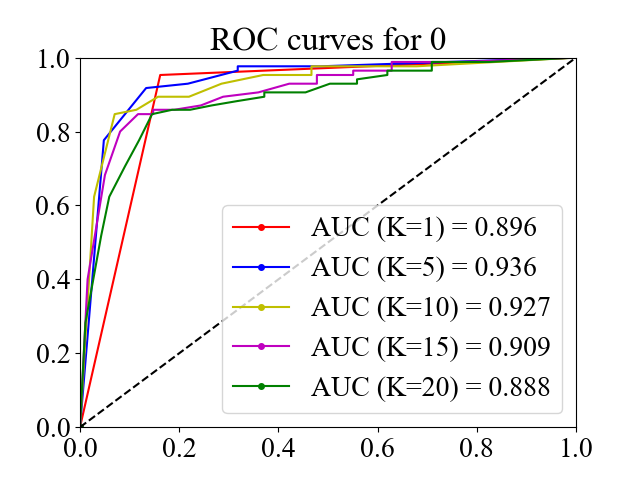}
\includegraphics[scale=0.19]{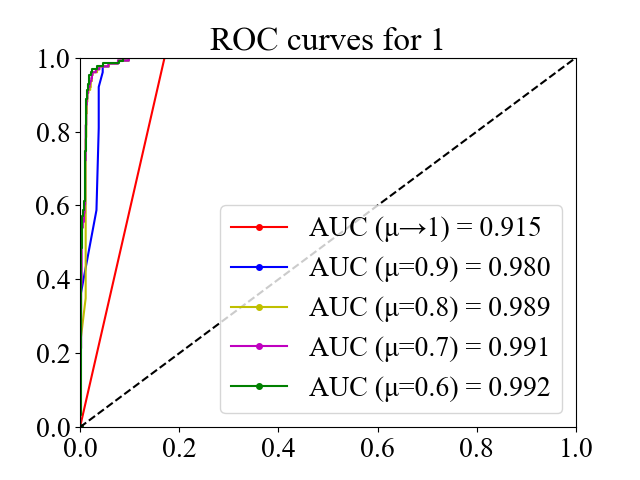}
\includegraphics[scale=0.19]{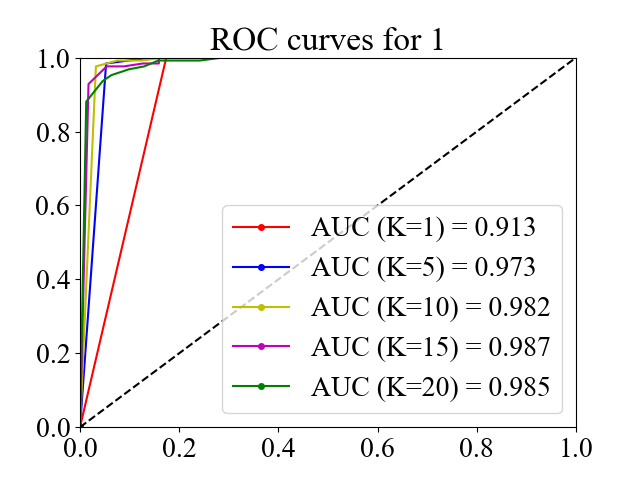}
\includegraphics[scale=0.19]{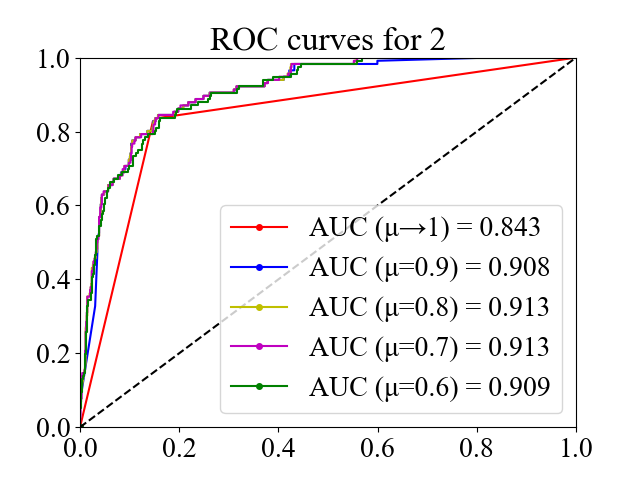}
\includegraphics[scale=0.19]{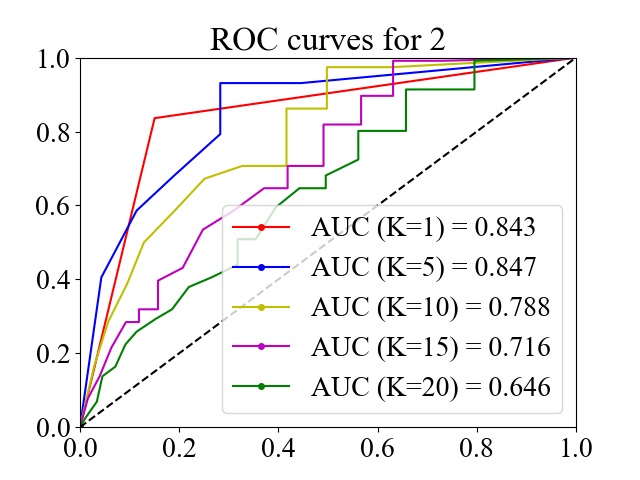}
\includegraphics[scale=0.19]{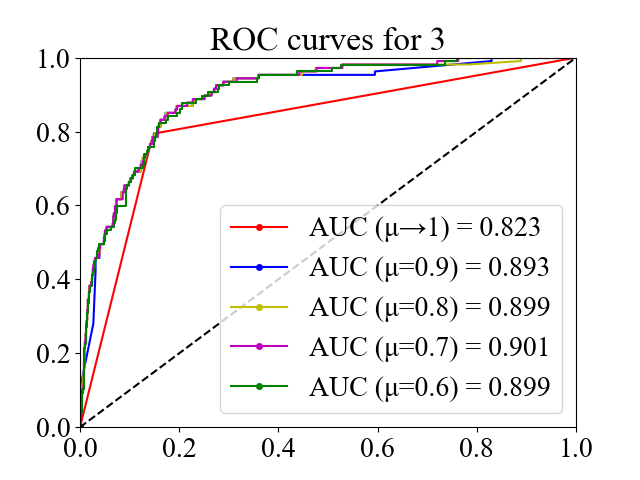}
\includegraphics[scale=0.19]{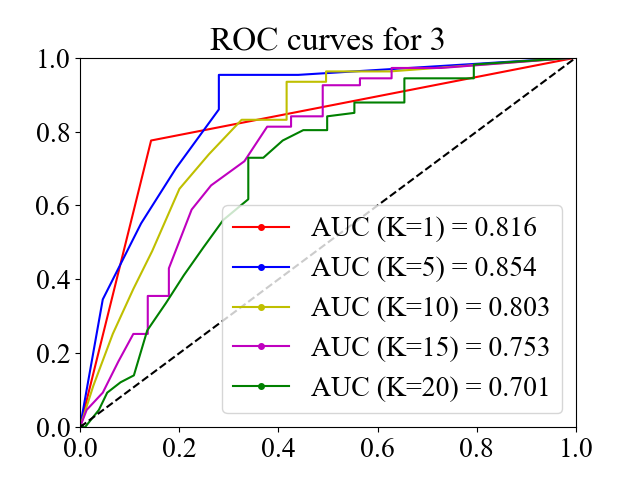}
\includegraphics[scale=0.19]{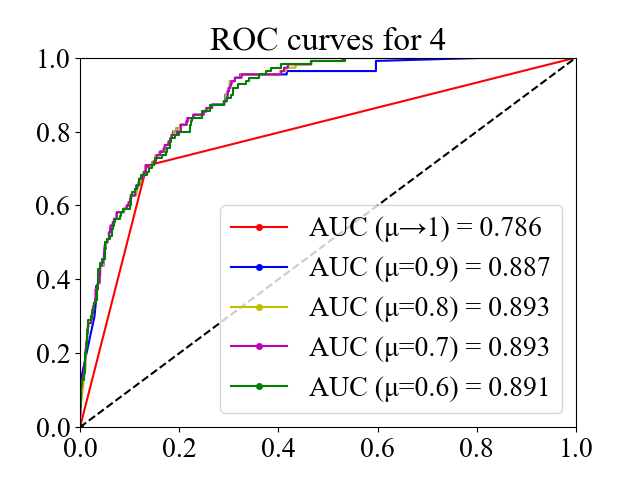}
\includegraphics[scale=0.19]{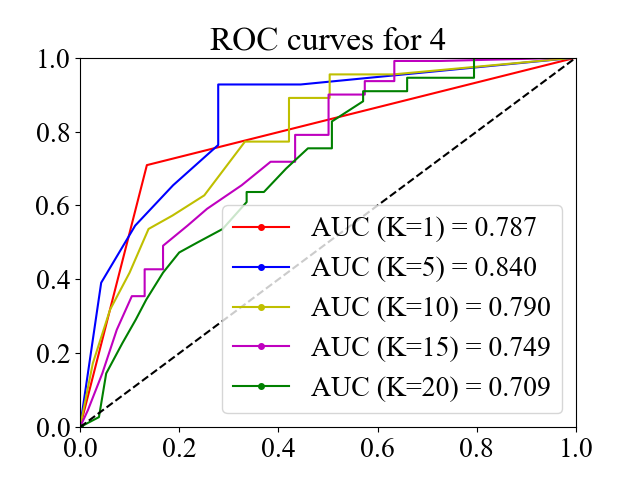}
\includegraphics[scale=0.19]{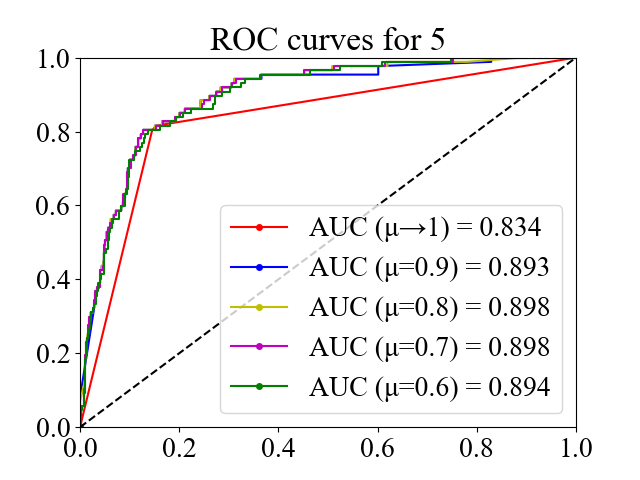}
\includegraphics[scale=0.19]{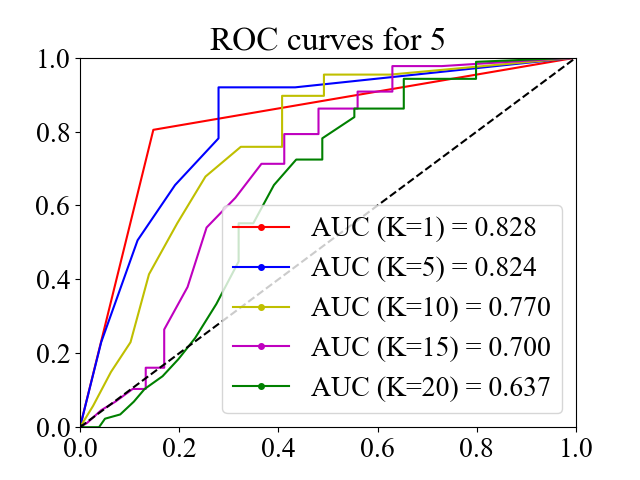}
\includegraphics[scale=0.19]{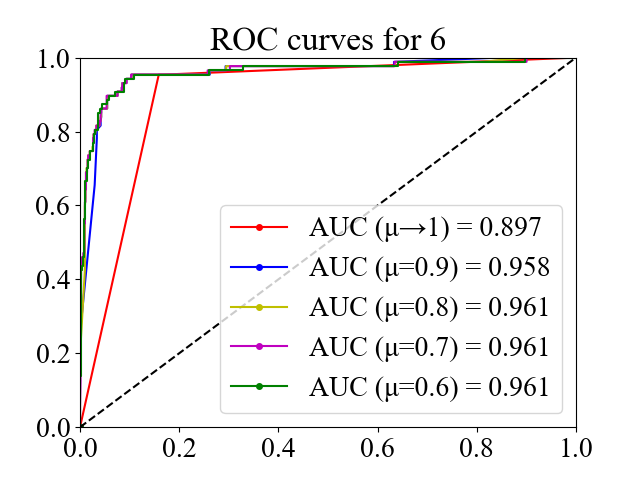}
\includegraphics[scale=0.19]{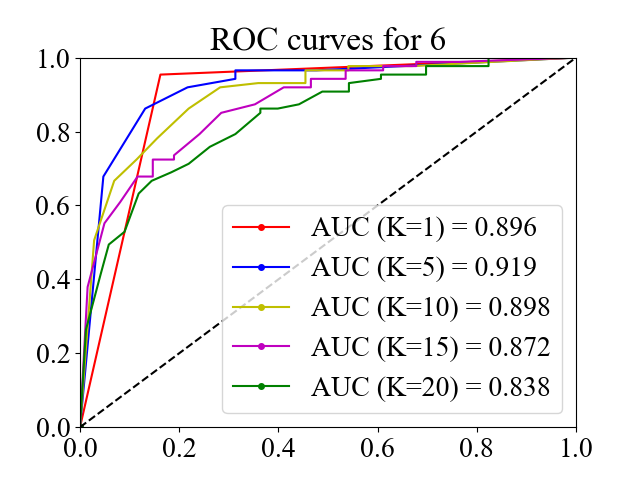}
\includegraphics[scale=0.19]{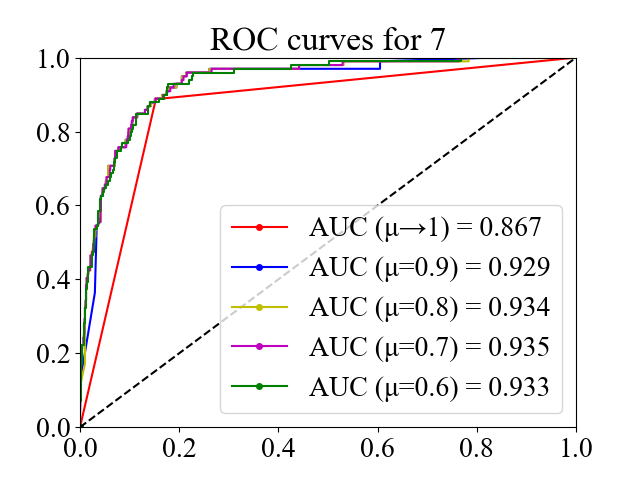}
\includegraphics[scale=0.19]{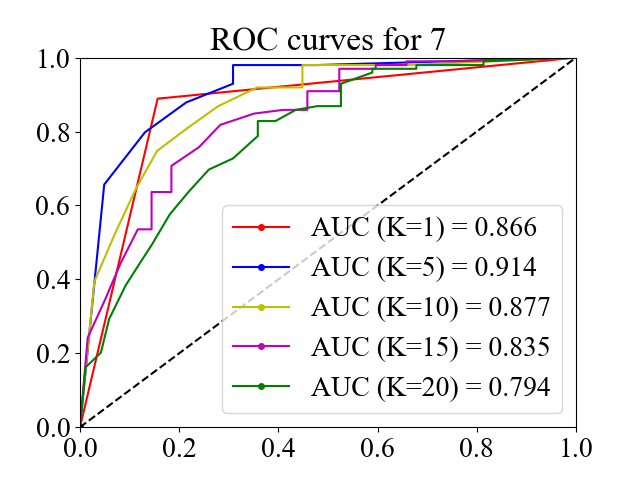}
\includegraphics[scale=0.19]{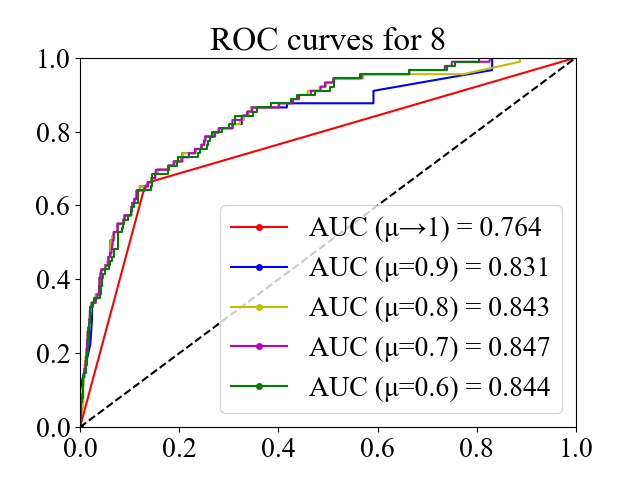}
\includegraphics[scale=0.19]{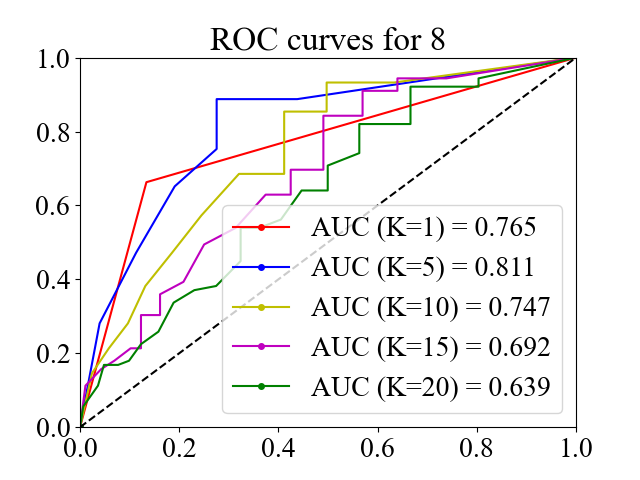}
\includegraphics[scale=0.19]{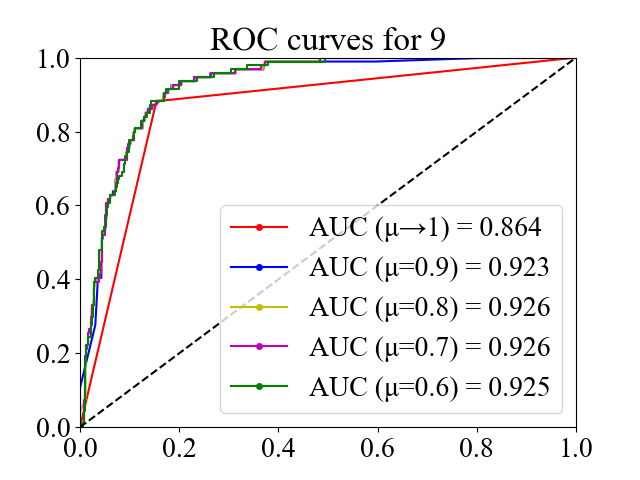}
\includegraphics[scale=0.19]{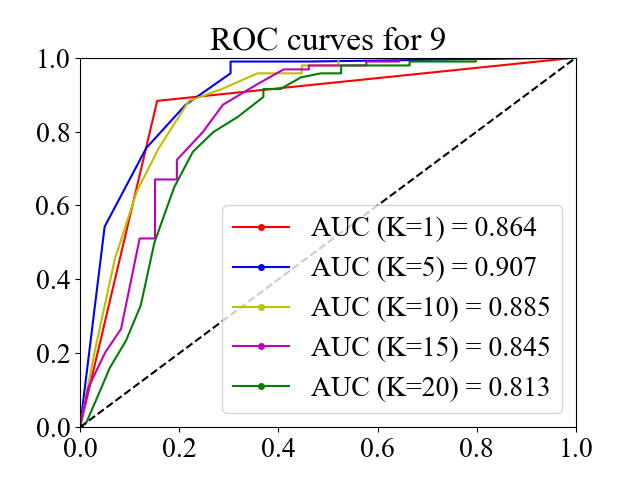}
\end{center}
\caption{ROC curves obtained using the GL (left per digit) and the K-NN method (right per digit). The first 1,000 images out of the 60,000 training images were used for training, and the first 1,000 images out of the 10,000 test images were used for testing.}
\label{fig:ROCs}
\end{figure}
\clearpage
\bibliographystyle{elsarticle-num}
\bibliography{btx_kido}
\end{document}